\newif\ifIEEE
\newif\ifCLASSOPTIONcompsoc
\newif\ifCLASSINFOpdf
\newif\ifCLASSOPTIONpeerreview
\newif\ifCLASSOPTIONcaptionsoff

\IEEEtrue

\ifIEEE
    \documentclass[journal]{IEEEtran}
    \CLASSOPTIONcompsocfalse 
    \CLASSOPTIONcaptionsofffalse 
    \CLASSOPTIONpeerreviewfalse 
    \CLASSINFOpdftrue 
\else
    \documentclass[twoside, 11pt]{article}
    \usepackage[abbrvbib]{jmlr2e}
    \jmlrheading{1}{2016}{1-48}{11/16}{10/00}{Andrew Gardner, Christian A. Duncan, Jinko Kanno, and Rastko R. Selmic}
    \ShortHeadings{On the Definiteness of Earth Mover's Distance}{Gardner et al.}
    \firstpageno{1}
    \CLASSOPTIONcompsocfalse 
\fi

\ifIEEE
	\ifCLASSOPTIONcompsoc
	  \usepackage[nocompress]{cite}
	\else
	  \usepackage{cite}
	\fi
	 \newcommand\mycite[1]{\cite{{#1}}}
	 \newcommand\mycitet[1]{\cite{{#1}}}
	 \newcommand\mycitealt[1]{\cite{{#1}}}
\else
	 \newcommand\mycite[1]{\citep{{#1}}} 
	 \newcommand\mycitet[1]{\citet{{#1}}}
	 \newcommand\mycitealt[1]{\citealt{{#1}}}
\fi
\ifIEEE
	\ifCLASSINFOpdf
	   \usepackage[pdftex]{graphicx}
	\else
	\fi
\fi
\usepackage[cmex10]{amsmath}
\ifIEEE
	\usepackage{amssymb}
	\usepackage{amsthm}
\else
\fi
\usepackage{xifthen}
\usepackage{mathtools}
\mathtoolsset{showonlyrefs=true}
\newcommand\myTitle{On the Definiteness of Earth Mover's Distance\\ and Its Relation to Set Intersection}

\newcommand\emd{\mathit{EMD}}
\newcommand\semd{\mathit{SEMD}}
\newcommand\emdhat{\widehat{\mathit{EMD}}}
\newcommand\emdnot{\overline{\mathit{EMD}}}
\newcommand\jemd{\mathit{JEMD}}
\ifIEEE
\newcommand\ie{i.e.\ }
\newcommand\eg{e.g.\ }

\else
\newcommand\ie{that is, }
\newcommand\eg{for example, }

\fi
\newcommand\emji{\mathit{EMJI}}
\newcommand\emi{\mathit{EMI}}
\newcommand\discrete{\delta_{\text{0-1}}}
\newcommand\supp{\mathrm{supp}}
\newcommand\dif{\mathrm{d}}
\renewcommand\vec[1]{\mathbf{#1}}
\renewcommand\Vec[1]{\mathrm{Vec}\left({#1}\right)}
\newcommand\uflow{H}
\newcommand\dens{\chi}

\usepackage[makeroom]{xcolor, cancel}
\newcommand\ptransform[2]{\ifthenelse{\isempty{#2}}{\mathrm{T}\left[{#1}\right]}{\mathrm{T}^{\left(#2\right)}\left[{#1}\right]}}

\newcommand\ntransform[2]{\overline{\mathrm{T}}_{{#2}}\left[{#1}\right]}

\newcommand\EMDHat[1][]{\ifthenelse{\isempty{#1}}{$\widehat{\text{EMD}}$}{$\widehat{\text{EMD}}_{#1}$}}
\newcommand\EMDNot[1][]{\ifthenelse{\isempty{#1}}{$\overline{\text{EMD}}$}{$\overline{\text{EMD}}_{#1}$}}
\newcommand\diag[1]{\mathrm{diag}\left({#1}\right)}
\newcommand\sign[1]{\mathrm{sign}\left({#1}\right)}
\renewcommand{\vec}[1]{\boldsymbol{\mathbf{#1}}}

\newcommand{\approptoinn}[2]{\mathrel{\vcenter{
  \offinterlineskip\halign{\hfil$##$\cr
    #1\propto\cr\noalign{\kern2pt}#1\sim\cr\noalign{\kern-2pt}}}}}

\newcommand{\suchthat}{\;\ifnum\currentgrouptype=16 \middle\fi|\;}

\ifIEEE
\newtheorem{theorem}{Theorem}
\newtheorem{lemma}{Lemma}
\newtheorem{proposition}{Proposition}
\newtheorem{corollary}{Corollary}
\newtheorem{conjecture}{Conjecture}

\newtheoremstyle{mytheoremstyle} 
    {\topsep}                    
    {\topsep}                    
    {}                   
    {}                           
    {\scshape\bfseries}                   
    {.}                          
    {.5em}                       
    {}  
\theoremstyle{mytheoremstyle}
\newtheorem{definition}{Definition}
\fi

\newcommand{\refreq}[1]{(\refeq{eq:#1})}

\newif\ifexpanded
\expandedtrue

\usepackage{acronym}
\acrodef{emd}[EMD]{earth mover's distance}
\acrodef{vbow}[VBOW]{Visual Bag-of-Words}
\acrodef{svm}[SVM]{support vector machine}
\acrodef{ksvm}[KSVM]{Krein support vector machine}
\acrodef{pd}[PD]{positive definite}
\acrodef{nd}[ND]{negative definite}
\acrodef{cpd}[CPD]{conditionally positive definite}
\acrodef{cnd}[CND]{conditionally negative definite}
\acrodef{rbf}[RBF]{radial basis function}
\acrodef{pde}[PDE]{partial differential equation}
\acrodef{emi}[EMI]{earth mover's intersection}
\acrodef{emjd}[EMJD]{earth mover's Jaccard distance}
\acrodef{emji}[EMJI]{earth mover's Jaccard index}
\acrodef{ospa}[OSPA]{optimal subpattern assignment}
\acrodef{idk}[IDK]{independence kernel}
\acrodef{pmk}[PMK]{pyramid match kernel}
\acrodef{iff}[iff]{if and only if}
\newcommand\posdef{\ac{pd}}
\newcommand\cposdef{\ac{cpd}}
\newcommand\negdef{\ac{nd}}
\newcommand\cnegdef{\ac{cnd}}
\newcommand\iffy{\acl{iff}}

\usepackage{mdwtab}
\usepackage{url}
\begin{document}
%
\title{\myTitle}
%
%
%
%

\ifIEEE
	\author{Andrew~Gardner,~\IEEEmembership{Student Member,~IEEE,}
	        Christian~A.~Duncan,
	        Jinko~Kanno,
	        and~Rastko~R.~Selmic,~\IEEEmembership{Senior Member,~IEEE}
	\IEEEcompsocitemizethanks{\IEEEcompsocthanksitem A. Gardner and J. Kanno are with the College of Engineering and Science, Louisiana Tech University, Ruston, LA, 71272.
	\protect\\
	E-mail: andrew.gardner1@ieee.org
	\IEEEcompsocthanksitem C. A. Duncan is with the Department of Mathematics and Computer Science at Quinnipiac University, Hamden, CT, 06518.
	\IEEEcompsocthanksitem R. R. Selmic is with the Department of Electrical and Computer Engineering at Concordia University, Montreal, Quebec, Canada.
	\protect\\
	E-mail: rastko.selmic@concordia.ca.
	}
	}
	
	%
	%

	\markboth{}
	{Gardner \MakeLowercase{\textit{et al.}}: \myTitle}
	%


\else
	\author{\name Andrew Gardner$^\dagger$ \email abg010@latech.edu
	        \AND
	        \name Christian~A.~Duncan \email christian.duncan@quinnipiac.edu\\
	\addr Department of Mathematics and Computer Science\\
	Quinnipiac University\\
	 Hamden, CT, 06518
	        \AND
	        \name Jinko~Kanno$^\dagger$ \email jkanno@latech.edu
	        \and
	        \name Rastko~R.~Selmic$^\dagger$ \email rselmic@latech.edu\\
	\addr $^\dagger$College of Engineering and Science\\
	 Louisiana Tech University\\
	  Ruston, LA, 71272}
\fi

\ifIEEE
\IEEEtitleabstractindextext{%
\fi

\begin{abstract}
Positive definite kernels are an important tool in machine learning that enable efficient solutions to otherwise difficult or intractable problems by implicitly linearizing the problem geometry.
In this paper we develop a set-theoretic interpretation of the Earth Mover's Distance (EMD) and propose Earth Mover's Intersection (EMI), a positive definite analog to EMD for sets of different sizes.
We provide conditions under which EMD or certain approximations to EMD are negative definite.
We also present a positive-definite-preserving transformation that can be applied to any kernel and can also be used to derive positive definite EMD-based kernels and show that the Jaccard index is simply the result of this transformation.
Finally, we evaluate kernels based on EMI and the proposed transformation versus EMD in various computer vision tasks and show that EMD is generally inferior even with indefinite kernel techniques.
\end{abstract}
\ifIEEE
	\ifCLASSOPTIONpeerreview
	\else
	\fi
	}
\else
	\editor{}
\fi

\maketitle

\ifIEEE
\IEEEdisplaynontitleabstractindextext

%
\IEEEpeerreviewmaketitle
\fi

\ifCLASSOPTIONcompsoc
\IEEEraisesectionheading{\section{Introduction}\label{sec:introduction}}
\else
\section{Introduction}
\label{sec:introduction}
\fi

%
%
%
%
\ifIEEE
\IEEEPARstart{T}{he} 
\else
The 
\fi
foundations of the \ac{emd}'s definiteness are the primary topic of this paper.
\ac{emd} is a metric that measures the minimum amount of one histogram that must be altered to transform it into another.
$\ac{emd}$ is commonly used in computer vision for comparing color distribution or texture histograms of images for content based image retrieval~\mycite{cg-emduts,rtg-emdmir,ho-eemdarhc,pw-lthmism,pw-fremd}.
If each histogram is represented by piles of dirt, \ac{emd} is the minimum cost required to move the dirt of one histogram until it acquires the distribution of the other, and from this interpretation its name---first used in print by Rubner et al.\@~\mycitet{rtg-emdmir}---naturally follows.
\ac{emd}, however, has a much longer history than its use in computer vision would imply.
Gaspard Monge~\mycitet{m-mstdr} originally laid the groundwork for \ac{emd}, and the problem was reformulated in the mid-20th century by Leonid Kantorovich~\mycitet{k-otm,k-opm}.
Thus does \ac{emd} receive its other name, the Monge-Kantorovich mass transportation distance, under which it is applied in economics, fluid mechanics, meteorology, and \acp{pde}~\mycite{e-pdemkmt,g-imtta}.
In statistics, \ac{emd} is also known as the Wasserstein or Mallows distance between probability distributions~\mycite{lp-emdmdsis,vib-drtdsd}.
The Wasserstein distance is also used as a means of evaluating the performance of multiple-object trackers and filters~\mycite{hm-mmdvoa,svv-cmpemof,rvcv-mpemtta}.
For a more comprehensive description of its history, the reader is referred to Vershik's recent article on the subject~\mycite{v-lhmktp}.
\ac{emd} is normally understood to be a discretized version of the Monge-Kantorovich distance, which is defined for continuous measures.
Regardless of the context, we will henceforth use \ac{emd} to refer to this distance in all of its various forms.

The \textit{ground distance} refers to the cost function chosen to define the distance between histogram bins.
\ac{emd} is usually assumed to possess a Euclidean ground distance, but examples of other ground distances exist in the literature. 
Igbida et al.\@\mycitet{imrt-mkmtpdd} study \ac{emd} in the context of \acp{pde} with a discretized version of the Euclidean ground distance rounded up to the nearest whole number.
Ling and Okada~\mycitet{ho-eemdarhc} proposed an efficient tree-based algorithm for computing \ac{emd} with a Manhattan ground distance, and Pele and Werman~\mycitet{pw-fremd} explored the effect of applying a threshold to various ground distances and its impact on computation time and accuracy.
In the realm of image retrieval, \ac{emd} is often applied as a metric for nearest neighbor searches.

\ac{emd} has also been applied in kernel methods for texture and object category classification with \acp{svm}~\mycite{zmls-lfkctoccs}.
\ifexpanded
A kernel is a function that possesses certain properties, namely symmetry and \posdef{}-ness, that can be used to efficiently solve certain nonlinear problems as though they are linear.
\fi
However, it is not known whether kernels derived from \ac{emd} are actually \posdef{}.
In fact, there is evidence to the contrary for a Euclidean ground distance~\mycite{ns-peinil1}.
Regardless, \ac{emd} continues to be used successfully for various purposes such as facial expression analysis~\mycite{szrs-wanaepmfeatpt} and EEG classification~\mycite{d-kemdeegc}.
Methods to ensure \posdef{}-ness have been explored~\mycite{zc-aascemdbkfusvmic}.
Cuturi~\mycitet{c-ptppdkh} suggested using the permanent of the transportation polytope, which is guaranteed to be \posdef{} although difficult to compute. 
Grauman and Darrell~\mycitet{gd-pmkelsf} \ifexpanded on the other hand \fi proposed a \posdef{} approximation of a maximum-cost version of \ac{emd} that also has the advantage of being easier to compute.

\subsection{Our Contributions}
In this paper we provide the following contributions:
\begin{itemize}
\item In Section~\ref{sec:pdpT}, we prove that the transformation \refreq{pdpT} is \posdef{}-preserving and generalizes the Tanimoto kernel of  Ralaivola et al.\@~\mycitet{rssb-gkci}. 
Coincidentally, we provide a new independent proof of the Jaccard index's \posdef{}-ness, which has already been the subject of at least two papers~\mycite{g-gcssip,bjd-ppdjim}.
As a corollary, we also deduce that the biotope transform~\mycite{dd-eod} preserves \cnegdef{}-ness in addition to metric properties.
\item In Section~\ref{sec:emi}, we propose \acl{emi}, a generalization of Pele and Werman's \EMDHat{}~\mycite{pw-lthmism} for kernels based on \ac{emd} with unnormalized sets. We show that given certain ground distances, \ac{emd} is \cnegdef{} and may thus be used to construct \posdef{} kernels using standard relations (\eg Lemma~\ref{thm:defK}).
\item We evaluate \EMDHat{} and its transformation using \refreq{pdpT} in a variety of classification experiments and show that both yield kernels superior to \ac{emd}, especially on unnormalized sets. 
The transformation in particular is shown to have some numerical advantages.
\end{itemize}

The next section presents relevant background material that may be used as a reference for the rest of the paper.
%

\section{Preliminaries}
This section provides definitions, lemmas, and other material that are useful for following the rest of the paper.
\subsection{Metrics}
A metric on a set $X$ is defined as follows.
\ifexpanded\else
We use the term \textit{discrete metric} to refer to the 0-1 distance defined by $\discrete(x,y) = 0$ if $x=y$ and 1 otherwise. \fi
\begin{definition}
A function $\delta: X \times X \to \mathbb{R}$ is a metric \iffy{} the following properties \ifexpanded are satisfied \else hold \fi for every $x,y,z \in X$.
\label{def:metric}
\ifexpanded
\begin{enumerate} 
\item Non-negativity: $\delta(x,y) \geq 0.$
\item Symmetry: $\delta(x,y) = \delta(y,x)$.
\item Identity of indiscernibles: $\delta(x,y) = 0$ \iffy{} $x=y$.
\item Triangle inequality: $\delta(x,y)\leq \delta(x,z)+\delta(y,z).$
\end{enumerate}
\else
(1) Non-negativity: $\delta(x,y) \geq 0.$
(2) Symmetry: $\delta(x,y) = \delta(y,x)$.
(3) Identity of indiscernibles: $\delta(x,y) = 0$ \iffy{} $x=y$.
(4) Triangle inequality: $\delta(x,y)\leq \delta(x,z)+\delta(y,z).$
\fi
\end{definition}
\ifexpanded
As can be inferred from its name, the discrete metric is a metric.
We also define the term \textit{semimetric} to indicate satisfaction of all of the preceding properties except for the triangle inequality.
The Euclidean distance is a metric, and the squared Euclidean distance is a semimetric.
A simple example of the squared Euclidean distance failing the triangle inequality may be noted with the points $x=(0,0)$, $y=(0,2)$, and $z=(0,1)$ as elements of $\mathbb{R}^2$.
We use the term \textit{discrete metric} to refer to the 0-1 distance defined by $\discrete(x,y) = 0$ if $x=y$ and 1 otherwise. 
\fi

\subsection{Kernels}
A \textit{kernel} on a set $X$ is, in general, a function $K: X \times X \to \mathbb{R}$.

\begin{definition}
A kernel $K$  is \textit{\posdef{}} \iffy{} it is symmetric and for any choice of $n$ distinct elements $x_1,\dots,x_n$ and real numbers $c_1, \dots, c_n$,
\begin{equation}
\sum_{i,j=1}^n c_ic_jK(x_i,x_j) \geq 0.\footnote{Be aware that our notation condenses the double summation when each index $i$ and $j$ shares the same range.}
\label{eq:defKEq}
\end{equation}
If the constraint $\sum_{i=1}^nc_i=0$ is added, then $K$ is \textit{\ac{cpd}}.
\label{def:defK}
\end{definition}
The condition \refreq{defKEq} is equivalent to testing whether the \textit{kernel matrix} for the chosen elements $G_K=[K(x_i,x_j)]$ is positive semi-definite via a quadratic form, \ie $\mathbf{c}^\intercal G_K\mathbf{c}\geq 0$ where $\mathbf{c}=\begin{bmatrix}c_1, \dots, c_n\end{bmatrix}^\intercal$.
A (conditionally) \textit{strictly \posdef{}} kernel is one in which the preceding inequalities are strict with equality holding only if each $c_i=0$.
One may note that \posdef{} implies \cposdef{}, but the converse does not hold.
\ifexpanded Simply reversing \else Reversing \fi the inequality of \refreq{defKEq} yields \textit{\negdef{}} kernels of each respective type.
Consequently, if $K$ is \posdef{}, then $-K$ is \negdef{}.
Note that we follow traditional nomenclature for kernels in that \posdef{} and strictly \posdef{} kernels correspond to positive \textit{semi}-definite and \posdef{} matrices, respectively. 
\posdef{} kernels are useful for a variety of machine learning tasks including classification, regression, and principal component analysis.

\posdef{}-ness is attractive because it implies the existence of a mapping $\phi:X \to H$ from $X$ to some Hilbert space $H$ in which the kernel gives the value of the inner product and certain nonlinear problems in $X$ become linear~\mycite{fm-kbamumatlab,stc-kmpa}, \ie
\ifexpanded
\begin{equation}
K(x_i,x_j) = \left<\phi(x_i),\phi(x_j)\right>.
\label{eq:innerproduct}
\end{equation}
\else
$K(x_i,x_j) = \left<\phi(x_i),\phi(x_j)\right>.$
\fi
This property is the key component of the so-called ``kernel trick" for \acp{svm}, wherein a separating hyperplane is implicitly found without ever working directly in $H$.
A \cnegdef{} kernel is also related to some Hilbert space $H$ via a mapping $\phi$ by
\ifexpanded
\begin{equation}
K(x_i,x_j) = \|\phi(x_i)-\phi(x_j)\|^2.
\label{eq:distance}
\end{equation}
\else
$K(x_i,x_j) = \|\phi(x_i)-\phi(x_j)\|^2.$
\fi
The existence of $\phi$ implies the respective type of definiteness and vice versa.
\ifexpanded
\cnegdef{} kernels are sometimes referred to as metrics of negative type, and as indicated by \refreq{distance}, correspond to functions that isometrically embed into squared Euclidean space.
\else
\cnegdef{} kernels are sometimes referred to as metrics of negative type and correspond to functions that isometrically embed into squared Euclidean space.
\fi

The following three results are adapted from \ifexpanded Berg et al.\@~\mycitet{bcp-haos} and form a basis for several later propositions. \else  Berg et al.\@~\mycitet{bcp-haos}. \fi
\ifexpanded
Theorem~\ref{thm:expK} and Lemma~\ref{thm:defK} propose relationships between \cnegdef{} and \posdef{} kernels.
\fi
The standard kernel used with \ac{emd} relies upon Theorem~\ref{thm:expK} with $u < 0$ and $K$ equal to \ac{emd} presumed to be \cnegdef{}.
\ifexpanded
Theorem~\ref{thm:productK}, originally proved by Schur~\mycitet{schur1911}, demonstrates that \posdef{} kernels are closed under multiplication.
\fi
Note that Theorem~\ref{thm:productK} does not apply to \cposdef{} kernels.
\begin{theorem}[\mycitealt{bcp-haos}]
Let $X$ be a nonempty set and let $K:X \times X \to \mathbb{R}$ be a symmetric kernel. 
Then $K$ is \cnegdef{} (\cposdef{}) \iffy{} $\exp(uK)$ is \posdef{} for each $u<0$ ($0<u$).
\label{thm:expK}
\end{theorem}
\begin{lemma}[\mycitealt{bcp-haos}]
Let $X$ be a nonempty set, $x_0 \in X$, and let $D:X \times X \to \mathbb{R}$ be a symmetric kernel.
Let $K(x,y):=D(x,x_0)+D(y,x_0)-D(x,y)-D(x_0,x_0)$.
Then $K$ is \posdef{} \iffy{} $D$ is \cnegdef{}.
If $D(x_0,x_0)\geq 0$, then $K_0(x,y) = K(x,y)+D(x_0,x_0)$ is also \posdef{}.
\label{thm:defK}
\end{lemma}
\begin{theorem}[\mycitealt{bcp-haos,schur1911}]
If $K_1: X\times X \to \mathbb{R}$ and $K_2:X \times X \to \mathbb{R}$ are both \posdef{}, then their Schur product $(K_1 \cdot K_2)(x,y) = K_1(x,y)K_2(x,y)$ is also \posdef{}.
\label{thm:productK}
\end{theorem}

The next two propositions are adapted from Boughorbel et al.\@~\mycitet{btb-ghikir} and were involved in the derivation of the generalized histogram intersection kernel. 
\ifexpanded
As a preview of upcoming proofs and an example of working with kernels, a proof of Proposition~\ref{thm:constantK} is given.
\fi
\begin{proposition}[\mycitealt{btb-ghikir}]
\begin{equation}
K_f(x,y) = f(x)+f(y)
\end{equation}
is both a \cposdef{} and \cnegdef{} kernel for any function $f$.
\label{thm:constantK}
\end{proposition}
\ifexpanded
\begin{proof}
Let $c_1,\dots,c_n$ and $x_1,\dots,x_n$ be defined as in Definition~\ref{def:defK} with $\sum_{i=1}^nc_i=0$.
\begin{equation}
\small
\begin{split}
\sum_{i,j=1}^nc_ic_jK_f(x_i,x_j) &= \sum_{i,j=1}^nc_ic_j\left[f(x_i)+f(x_j)\right]\\
&=\sum_{i,j=1}^nc_ic_jf(x_i)+\sum_{i,j=1}^nc_ic_jf(x_j)\\
&=2\sum_{i,j=1}^nc_jc_if(x_i)\\
&=2\left(\sum_{j=1}^nc_j\right)\left(\sum_{i=1}^nc_if(x_i)\right)\\
&=0.
\end{split}
\end{equation}
\end{proof}
\fi
\begin{proposition}[\mycitealt{btb-ghikir}]
If $K$ is positive valued and a \cnegdef{} kernel, then $K^{-\gamma}$ is \posdef{} for each $\gamma \geq 0$.
\label{thm:inverseK}
\end{proposition}

\subsection{Measures and Multisets}
A \textit{measure} is a function that generalizes the notion of cardinality, area, \ifexpanded volume, or length. \else or volume. \fi
To be precise, a measure $\mu: \Sigma_X \to \mathbb{R}$ assigns a number to subsets contained in a $\sigma$-algebra $\Sigma_X$ of some set $X$.
The measure of a subset must be less than or equal to that of its superset, \ie $\mu(A) \leq \mu(B)$ if $A \subseteq B$.
Measures also possess countable additivity, \ie the measure of the union of disjoint sets is the sum of their measures.
For the remainder of the paper, we assume that sets are non-negative and finitely measurable.
\ifexpanded
A \textit{measure space} $(X, \Sigma, \mu)$ is a \textit{measurable space} $(X,\Sigma)$ paired with a measure $\mu$.
Cardinality is sometimes referred to as the counting or discrete measure.
\fi
We will use the terms mass, size, and measure interchangeably to denote the value of $\mu(A)$.

A \textit{multiset} generalizes a set by allowing duplicate elements.
We use the terms multiset and set interchangeably with context indicating which is meant in the strict sense.
By definition, the multiplicity of an element $x$ is a non-negative integer indicating how many copies of $x$ are contained in a given multiset.
\ifexpanded
We generalize this definition by allowing a non-negative real number of ``copies." 
With this definition, we may also include probability distributions and other continuous functions with real output.
\else
We generalize this definition by allowing a non-negative real number of ``copies" so that we may also include probability distributions and other continuous functions.
\fi

Let $X$ be the set of all possible elements under consideration.
Let $\dens_A(x)$ be the mass density (or multiplicity) function of the multiset $A \subseteq X$ with $x \in X$.
The density function completely defines a multiset.
When we refer to one, the other is implied.
Note that for a standard set $A$ (\ie not multiset), $\dens_A(x)$ is the characteristic or indicator function of $A$ returning 1 for $x\in A$.
For any element $x$ not contained in $A$, $\dens_A(x) = 0$.
The mass density function of $A$ gives rise to a measure 
\begin{equation}
\mu_A(Y) = \int_Y \dens_A(x) \dif\mu_A(x).
\label{eq:mudef}
\end{equation}
For discrete sets, \refreq{mudef} simplifies to series summation.
The membership of an element $x \in A$ is contingent upon $\dens_A(x) > 0$, and the \textit{support} of a multiset $\supp(A)$ is the set of all elements $x \in X$ for which $\dens_A(x) \ne 0$.
We use the term singleton to denote a multiset $A$ with support satisfying
$
\supp(A) = \left\{x_0\right\}
$
for some fixed element $x_0 \in A$.

We \ifexpanded generalize the definition of \else define \fi a subset $A \subseteq B$ in $X$ to be such that $\dens_A(x) \leq \dens_B(x)$ for each $x \in X$.
The density function for the intersection of two multisets $A$ and $B$ is defined as
\begin{equation}
\dens_{A \cap B}(x) = \min \left\{ \dens_A(x), \dens_B(x) \right\},
\label{eq:intersection}
\end{equation}
and the union is similarly defined with $\max$ instead of $\min$\ifexpanded:
\begin{equation}
\dens_{A\cup B}(x) = \max \left\{\dens_A(x), \dens_B(x) \right\}.
\label{eq:union}
\end{equation}
We also define the sum of two sets as
\begin{equation}
\dens_{A+ B}(x) = \dens_A(x)+ \dens_B(x).
\label{eq:sumSet}
\end{equation}
\else.
\fi

Define $\mathcal{F}(X)$ to be the family of multisets $A$ with support $\supp(A) \subset X$ yielding finite, non-negative measure $\mu_A(X)$, and let $\mathcal{P}(X) \subset \mathcal{F}(X)$ be the family of all multisets $A$ with $\mu_A(X)=1$ (\ie probability distributions).
Henceforth, we will abuse notation by defining
$
\mu(A) = \mu_A(\supp(A)).
$
Unlike histograms, multisets do not imply a finite, countable base set  $X$ from which every set draws its support. 
This distinction allows somewhat more flexible definitions of \ac{emd}.

\subsection{Earth Mover's Distance}
We consider \ac{emd} to be a metric\footnote{Here we mean metric as in dissimilarity measure. Note that \ac{emd} is not a true metric in the sense of Definition~\ref{def:metric} on $\mathcal{F}(X)$ but rather on $\mathcal{P}(X)$ for metric ground distance~\mycite{rtg-emdmir}. \ifexpanded Violations of identity and triangle inequality are easily found when considering subsets and supersets.\fi} on $\mathcal{F}(X)$ for some set $X$.
Application of \ac{emd} requires specification of a ground distance $D: X \times X\to\mathbb{R}$ and computation of the \textit{flow} $f(a,b)$ of mass from $a \in A$ to $b \in B$ with $A,B \in \mathcal{F}(X)$.
Since any ground distance is just a special case of \ac{emd} between singletons of unit mass, \ac{emd} is \cnegdef{} only if the ground distance is \cnegdef{}. 
Unless otherwise noted, we will assume that any ground distances discussed henceforth are \cnegdef{}.
\ac{emd} may then be defined as the solution of the following linear programming problem, which calculates the cost of the minimum-cost maximum flow:
\begin{equation}
\emd(A,B) =\underset{f}{\min} \sum_{a\in A} \sum_{b \in B} f(a,b)D(a,b),
\label{eq:emd}
\end{equation}
subject to the constraints
\begin{equation}
\sum_{b \in B} f(a,b) \leq \dens_A(a),
\label{eq:outFlow}
\end{equation}
\begin{equation}
\sum_{a \in A} f(a,b) \leq \dens_B(b),
\label{eq:inFlow}
\end{equation}
\begin{equation}
\sum_{a \in A}\sum_{b \in B} f(a,b) = \min\left\{ \mu(A), \mu(B)\right\}.
\label{eq:totalflow}
\end{equation}
For convenience, we have defined \ac{emd} with the assumption that $A$ and $B$ both have discrete support.
Constraints \refreq{outFlow} and \refreq{inFlow} state that the amount of each element transported is limited by the available mass in each set located at that element.
By similar reasoning, note that the sum of the overall flow is constrained to be equal to the mass of the smaller set \refreq{totalflow}, effectively forcing the transportation of the maximum possible amount of mass.
\ac{emd} has always been practically limited by the computational complexity involved in solving the linear program, although there has been recent development of fast approximation algorithms~\mycite{c-sdlcot,sgpcbndg-cwdeotgd}.

Note that our definition of \ac{emd} differs slightly from that of Rubner et al.\@~\mycitet{rtg-emdmir}, which scales \refreq{emd} by the inverse of the total flow in \refreq{totalflow}.
For sets of the same size, Rubner's definition is just \refreq{emd} scaled by a constant factor.
Pele and Werman~\mycitet{pw-lthmism} introduced a means to calculate \ac{emd} between unnormalized histograms for use in nearest neighbor calculations and image retrieval:
\begin{equation}
\small
\begin{split}
\emdhat_{\alpha}(A,B) &= \emd(A,B) + \alpha|\mu(A)-\mu(B)|\underset{a,b\in X}\max\{D(a,b)\},
\label{eq:originalEMDHat}
\end{split}
\end{equation}
where $\alpha \geq 0$ and $D$ is presumed to be bounded.
\EMDHat[\alpha] is a metric on $\mathcal{F}(X)$ if \ac{emd} is a metric on $\mathcal{P}(X)$ and $\alpha \geq 0.5$~\mycite{pw-lthmism}.
Schuhmacher et al.\@~\mycitet{svv-cmpemof} independently proposed an almost identical version of \EMDHat{} under the acronym OSPA (Optimal Subpattern Assignment).

Normalized forms of \ac{emd} have also been proposed by Gardner et al.\@~\mycitet{CVPR2014} and by Ramon and Bruynooghe~\mycitet{rb-ptcmbps}, although the latter did not acknowledge a connection to \ac{emd}.
The transformation of the following section was inspired by the search for and study of a normalized form.


\section{A Definite-Preserving Transformation}\label{sec:pdpT}
In this section we propose the \posdef{}-preserving transformation
\begin{equation}
\ptransform{K}{}(x,y) = \frac{K(x,y)}{K(x,x)+K(y,y)-K(x,y)},
\label{eq:pdpT}
\end{equation}
that normalizes any given \posdef{} kernel $K$.
If $K(x,x)=K(y,y)=0$, we define $\ptransform{K}{}(x,y)=1$.
As opposed to the traditional normalization,
\ifexpanded
\begin{equation}
K_N(x,y) = \frac{K(x,y)}{\sqrt{K(x,x)K(y,y)}},
\label{eq:normT}
\end{equation}
\else
${K_N(x,y) = K(x,y)/\sqrt{K(x,x)K(y,y)}}$,
\fi
which can be interpreted as a surjective mapping of images $\phi(x)$ in Hilbert space onto the unit hypersphere via projection, $\ptransform{K}{}$ can be interpreted as an injective mapping onto a unit hypersphere of unspecified dimension.
Image vectors in Hilbert space of different magnitude that share the same direction remain distinguishable post-transformation.

Technically, this kernel (or one algebraically equivalent to it) has been proposed before as the Tanimoto kernel by Ralaivola et al.\@~\mycitet{rssb-gkci}.
We stress the differences in our proposed transformation and how our contributions differ from existing work.
First, the Tanimoto kernel is equivalent to the Jaccard index and has only been proved \posdef{} when $X$ consists solely of binary vectors and $K$ is the dot product (see the proof given by Ralaivola et al., which hinges on the proof of semi-\posdef{}-ness of the Jaccard index given by Gower~\mycitet{g-gcssip}).
We prove (see Theorem~\ref{thm:pSmoothK}) that~\refreq{pdpT} is strictly \posdef{} for any $K$ if $K$ is strictly \posdef{} (and similarly for semi-definiteness), which is stronger than the proof of Ralaivola et al.\@ and more general than both it and the proof of strict \posdef{}-ness of the Jaccard index given by Bouchard et al.\@~\mycitet{bjd-ppdjim}.
Since we are not limited to binary vectors, the range of~\refreq{pdpT} is not even constrained to be positive.
This more general view of the transformation also allows us to examine its properties in new situations, such as when it is applied to itself or \textit{nested}.

In fact, the transformation can be nested indefinitely as in 
\begin{equation}
\ptransform{K}{n}(x,y) = \ptransform{\ptransform{K}{n-1}}{}(x,y),
\label{eq:nestedTransformation}
\end{equation}
where $\ptransform{K}{n}$ is the $n$-th nested transformation with ${\ptransform{K}{0} \equiv K}$.
From~\refreq{pdpT}, one should see that $\ptransform{K}{}(x,x)=1$ for any $x$ and $K$.
Consequently, we note that for $n \geq 1$,
\begin{equation}
\ptransform{K}{n+1}(x,y)=\frac{\ptransform{K}{n}(x,y)}{2-\ptransform{K}{n}(x,y)}.
\label{eq:fixedPointKT}
\end{equation}
The following two propositions describe the limiting behavior of nested transformations. 
See Appendix~\ref{sec:nestedT} for proofs and a closed form expression of $\ptransform{K}{n}$.
\begin{proposition}
For any \posdef{} kernel $K$,  
\begin{equation}
\lim_{n \to \infty} \ptransform{K}{n}(x,y) = \begin{cases}
0 & \text{if } x \ne y,\\
1 & \text{otherwise.}
\end{cases}
\label{eq:pdpTLimits}
\end{equation}
\label{thm:pdpTLimits}
\end{proposition}
\begin{proposition}
For any symmetric kernel $K: X \times X \to \mathbb{R}$ where $X$ is finite and  $x \ne y \implies 2K(x,y) \ne K(x,x)+K(y,y)$ for $x,y\in X$, there exists a number $n_0$ such that $\ptransform{K}{n}:X \times X \to \mathbb{R}$ is \posdef{} for all $n \geq n_0$.\ifexpanded\footnote{
We hypothesize that the proposition holds simply if $X$ is finite and $K$ satisfies the equivalence relation 
\begin{equation}
x \sim y \iff 2K(x,y) = K(x,x)+K(y,y).
\end{equation}}
\fi
\label{thm:inducedPDPT}
\end{proposition}

We now show that the transformation preserves \ifexpanded definiteness as claimed. \else definiteness. \fi

\begin{theorem}
If $K:X \times X \to \mathbb{R}$ is \posdef{}, then the function $\ptransform{K}{}$ as defined by \refreq{pdpT}
is also \posdef{}.
\label{thm:pSmoothK}
\end{theorem}
\begin{proof}
Without loss of generality, assume $K(x,x)=K(y,y)=0 \implies x=y = p$ for some $p \in X$ and let us restrict $K$ in the following discussion to $X \setminus \left\{p\right\}.$
The denominator in \refreq{pdpT} is positive valued due to a well-known property of \posdef{} kernels and matrices,
\begin{equation}
\left|K(x,y)\right| \leq \frac{K(x,x)+K(y,y)}{2} < K(x,x)+K(y,y).
\label{eq:posDefTriangle}
\end{equation}
The denominator is also \cnegdef{} as it is the sum of two \cnegdef{} kernels: $K(x,x)+K(y,y)$ (by Proposition~\ref{thm:constantK}) and $-K(x,y)$ (by hypothesis).
Thus by Proposition~\ref{thm:inverseK} with $\gamma=1$,
\begin{equation}
K_1(x,y) = [K(x,x)+K(y,y)-K(x,y)]^{-1}
\end{equation}
is \posdef{}.
We therefore have the product of two \posdef{} kernels
\begin{equation}
\begin{split}
\ptransform{K}{}(x,y) = K(x,y)K_1(x,y),
\end{split}
\end{equation}
which is itself \posdef{} by Theorem~\ref{thm:productK}.

In order to include the case $x=y=p,$ we note that if $\phi: X \to H$ is the kernel's feature mapping into the Hilbert space $H$, then 
$K(p,p) = \left<\phi(p),\phi(p)\right> = 0 \implies \phi(p) = \vec{0},$ which further implies 
\begin{equation}
K(p,x_i) = \left<\phi(p), \phi(x_i)\right> =\left<\vec{0},\phi(x_i)\right> = 0
\end{equation}
for $x_i \ne p$.
Therefore, $\ptransform{K}{}(x_i,p) = 0$ if $x_i \ne p$.
Let $x_0=p$ and $c_0 \in \mathbb{R}$.
Then $\ptransform{K}{}$ is \posdef{} because
\begin{equation}
\begin{split}
\sum_{i,j=0}^n c_ic_j\ptransform{K}{}(x_i,x_j) 
&=c_0^2+\sum_{i,j=1}^n c_ic_j\ptransform{K}{}(x_i,x_j)\geq 0.
\end{split}
\end{equation}
\end{proof}
\begin{corollary}
Let $D: X \times X \to \mathbb{R}$ be a \cnegdef{} kernel, and let $p \in X$.
Then,
\begin{equation}
\ifIEEE
\footnotesize
\fi
\begin{split}
\hspace{-1em}\ntransform{D}{p}(x,y) = \frac{2D(x,y)-D(x,x)-D(y,y)}{D(x,p)+D(y,p)+D(x,y)-\underset{z \in \{x,y,p\}}{\sum}D(z,z)},
\end{split}
\label{eq:cndpT}
\end{equation}
is also \cnegdef{}.
\label{thm:biotopeD}
\end{corollary}
\begin{proof}
We can define a \posdef{} kernel $K_p$ according to the relation given by Lemma~\ref{thm:defK}, \ie
\begin{equation}
K_p(x,y) = D(x,p)+D(y,p)-D(x,y)-D(p,p).
\label{eq:pCenteredK}
\end{equation}
Using \refreq{pCenteredK}, note that $K_p(x,x)=2D(x,p)-D(x,x)-D(p,p)$.
Furthermore, note that
\ifIEEE
\begin{equation}
\small
\begin{split}
K_p(x,x)+K_p(y,y)-K_p(x,y)
\hspace{-0.2em}&=\hspace{-0.2em}D(x,p)+D(y,p)+D(x,y)\\
&\quad\hspace{-0.2em}-D(x,x)-D(y,y)-D(p,p).
\end{split}
\end{equation}
\else
\begin{equation}
\small
\begin{split}
K_p(x,x)+K_p(y,y)-K_p(x,y)
&=D(x,p)+D(y,p)+D(x,y)-D(x,x)-D(y,y)-D(p,p).
\end{split}
\end{equation}
\fi
We see that the denominator of $\ntransform{D}{p}$ is the same as that of $\ptransform{K_p}{}$.
Note then that 
\begin{equation}
\ptransform{K_p}{}(x,y)+\ntransform{D}{p}(x,y) = 1.
\label{eq:complementary}
\end{equation}
If $x_1, \dots, x_n \in X$, $c_1, \dots, c_n \in \mathbb{R}$, and $\sum_{i=1}^n c_i=0$, then
\ifexpanded
\begin{equation}
\begin{split}
\sum_{i,j=1}^nc_ic_j\ntransform{D}{p}(x_i,x_j) &= \sum_{i,j=1}^nc_ic_j\left(1-\ptransform{K_p}{}(x_i,x_j)\right)\\
&= \sum_{i,j=1}^nc_ic_j\left(-\ptransform{K_p}{}(x_i,x_j)\right)\leq 0.
\end{split}
\end{equation}
\else
\begin{equation}
\begin{split}
\sum_{i,j=1}^nc_ic_j\ntransform{D}{p}(x_i,x_j) = \sum_{i,j=1}^nc_ic_j\left(-\ptransform{K_p}{}(x_i,x_j)\right)\leq 0.
\end{split}
\end{equation}
\fi
We have thus shown that $\ntransform{D}{p}$ is \cnegdef{}.
\end{proof}

If $K(x,y) \geq 0$, then $\ptransform{K}{}(x,y) \in [0,1]$.
Otherwise, $\ptransform{K}{}(x,y) \in [-1/3, 1]$.
Consequently, $\ntransform{D}{p}(x,y) \in [0,4/3]$ and $\ntransform{D}{p}(x,y) > 1$ \iffy{} $D(x,y)+D(p,p) > D(x,p)+D(y,p)$. 
In addition, Theorem~\ref{thm:pSmoothK} also holds for strictly \posdef{} $K$.
Using Theorem~\ref{thm:pSmoothK} with $K$ as the intersection kernel therefore provides an easy proof for the \posdef{}-ness of the Jaccard index,
\ifexpanded
\begin{equation}
J(A,B) = \frac{\mu(A \cap B)}{\mu(A \cup B)}.
\end{equation}
\else
$J(A,B) = \mu(A \cap B)/\mu(A \cup B).$
\fi
Note that $\ntransform{D}{p}$ generalizes the well-known biotope transform~\mycite{dd-eod}, showing that it preserves  \cnegdef{}-ness in addition to metric properties.
As an example, suppose $A$ and $B$ are sets and $D(A,B) = |\mu(A)-\mu(B)|$. 
This kernel is \cnegdef{}.
By Corollary~\ref{thm:biotopeD} with $p=\emptyset$ followed by some simplification, we can conclude that the following is \cnegdef{}:
\begin{equation}
\begin{split}
\ntransform{D}{\emptyset}(A,B) = \frac{|\mu(A)-\mu(B)|}{\max\left\{\mu(A),\mu(B)\right\}}.
\end{split}
\end{equation}

\section{Earth Mover's Intersection: A Set Theoretic Interpretation of \ac{emd}}\label{sec:emi}
In this section we introduce \ac{emi}, a useful concept and \posdef{} analog to \ac{emd} that computes the similarity between two sets rather than their difference for a given ground distance.
The name comes from the following motivating scenario.

Suppose there are two sets of two-dimensional points where one is a slightly perturbed version of the other.
According to the strict definition of set intersection given by \refreq{intersection}, their intersection is empty despite the fact that they are clearly related by their elements.
The inability of set intersection to account for the sets' inherent similarity is a problem.
\ac{emd} provides a natural solution to this problem, although it is proportional to the sets' difference rather than similarity.
\ac{emd} also reflects the qualities of whatever norm is chosen to compare the individual points.
We now show that \ac{emd} and subsequent related functions define smooth (in the sense of strictness of equality)  generalizations or approximations of classic set operations.

Sets are usually normalized prior to application of \ac{emd} by dividing their density function by their total mass, an operation analagous to normalizing a vector to unit norm. 
The disadvantage of this method is that sets with differently scaled but otherwise identical density functions become indistinguishable post-normalization.
As a side-effect, one removes an entire dimension of the data (for the most extreme case, consider singleton point sets with non-negative mass on the real line).
An application where this distinction is important is that of multi-object tracking and filtering~\mycite{svv-cmpemof,rvcv-mpemtta}; normalizing set mass can cause one to ignore the fact that the incorrect number of objects are being tracked.
For our set theoretic interpretation of \ac{emd}, we prefer to retain the sets' original mass and transport excess mass to a predetermined point $p \in X$.
One could also consider this a form of additive normalization by supplementing mass at the point $p$.
\ac{emd} then more accurately represents the relative magnitudes of set differences as well as distinguishes differently scaled sets.

Define the term $\emdnot_{p}$ to represent the transportation of excess mass from the larger of two sets $A$ and $B$ to some sink $p \in X$:
\begin{equation}
\ifIEEE
\footnotesize
\fi
\begin{split}
\hspace{-1.5em}\emdnot_{p}(A,B)=& \sum_{b \in B}\left(\dens_B(b)\hspace{-0.25em}-\hspace{-0.5em}\sum_{a \in A}f^*(a,b)\right)\left[D(b,p)\hspace{-0.25em}-\hspace{-0.25em}\frac{D(p,p)}{2}\right]\hspace{-0.25em},
\end{split}
\label{eq:emdnot}
\end{equation}
where $D$ is the ground distance, $f^*$ is the optimal flow, and we assume without loss of generality that $\mu(A) \leq \mu(B)$.
The total cost of transforming one set into another is then given by
\begin{align}
\begin{split}
\emdhat_{p}(A,B) &= \emd(A,B) + \emdnot_p(A,B),
\end{split}
\label{eq:generalizedEMDHat}
\end{align}
where we have adopted the notation for Pele and Werman's \EMDHat[\alpha]. 
Note that $p$ does not necessarily have to be in $X$ (in which case we must replace $D$ with an appropriate function in \refreq{emdnot}).
Ideally, though, $p$ is a reserved point that does not naturally appear in the sets under consideration.
Otherwise, there is a different type of potential identity loss.

We define \ac{emi} as the kernel resulting from Lemma~\ref{thm:defK} with $x_0=\emptyset$ and $D=\emdhat_p$:
\begin{equation}
\ifIEEE
\small
\fi
\begin{split}
\emi_p(A,B) = & \emdhat_p(A, \emptyset)+\emdhat_p(B,\emptyset)-\emdhat_p(A,B).
\end{split}
\end{equation}
Note that \ac{emi} is  \posdef{} whenever \EMDHat{} is \cnegdef{} for some collection of sets (and vice versa).
By assuming $p \in X$, we can define a \posdef{} kernel $K_p$ according to Lemma~\ref{thm:defK} with $x_0=p$ and $D$ as the ground distance, which we can then use with \refreq{emdnot} to simplify \ac{emi} to
\begin{equation}
\begin{split}
\emi_p(A,B) =& \sum_{a \in A}\sum_{b \in B}f^*(a,b)K_p(a,b).
\label{eq:emi}
\end{split}
\end{equation}
Observe that the minimum-cost maximum flow with respect to $D$ is the same as the maximum-cost maximum flow with respect to $K_p$, regardless of the choice of $p$.
As a result, \ac{emi} can be specified in terms of just a \posdef{} ground distance without explicitly specifying $p$.
The definition of \ac{emi} also provides some insight into the \ac{pmk}~\mycite{gd-pmkelsf}, which can be viewed as an approximation of $\emi_{\vec{0}}$ on $\mathcal{F}(\mathbb{R}^n)$.
One may also consider an alternative definition $\emi^\prime_p(A,B) = \emi_p(A,B)+\sum\sum f^*(a,b)D(p,p)$ that is also \posdef{} if  $D(p,p) \geq 0$ and \EMDHat{} is \cnegdef{}; this is equivalent to discarding $D(p,p)$ in \refreq{emdnot}.

As our first example of a situation in which \ac{emi} is \posdef{} on $\mathcal{F}(X)$ (and hence \EMDHat{} and \ac{emd} are respectively \cnegdef{} on $\mathcal{F}(X)$ and  $\mathcal{P}(X)$), consider the discrete metric, which can trivially be verified to be \cnegdef{}.
Define the \textit{discrete kernel} corresponding to this ground distance to be
$
K_{\text{0-1}}(x,y) = 1- \discrete(x,y),
$
which is \posdef{}.
We can show that \ac{emi} in this case is equivalent to the intersection kernel.
\begin{proposition}
Let $\emi_{\text{0-1}}(A,B)$ be \ac{emi} equipped with the discrete kernel as the ground distance on an arbitrary set $X$. Then $\emi_{\text{0-1}}$ is equivalent to the intersection kernel.
\end{proposition}
\begin{proof}
The goal is to find the maximum-cost maximum flow subject to constraints, and the only way to increase the cost with the discrete kernel is to send available mass from a point in one set up to the capacity allowed by the other set at the same location.
Therefore, $f^*(a,a)$ will be saturated up to the available capacity at $a$ in each set, \ie
\begin{equation}
f^*(a,a) = \min\left\{\dens_A(a),\dens_B(a)\right\}.
\label{eq:inclusionFlow}
\end{equation}
The cost to transport this mass is simply the amount of mass transported.
The exact mapping of the remaining mass is irrelevant as it costs nothing to move.
As a result,
\ifexpanded
\begin{equation}
\begin{split}
\emi_{\text{0-1}}(A,B) =& \sum_{a \in A \cup B} f^*(a,a) \\
=& \sum_{a \in A \cup B}\min\left\{\dens_A(a),\dens_B(a)\right\}\\
=& \mu(A \cap B).
\end{split}
\end{equation}
\else
\begin{equation}
\small
\begin{split}
\emi_{\text{0-1}}(A,B) =& \sum_{a \in A \cup B} f^*(a,a) = \sum_{a \in A \cap B} f^*(a,a) =\mu(A \cap B).
\end{split}
\end{equation}
\fi
\end{proof}
Since the intersection kernel is \posdef{}~\mycite{btb-ghikir}, we conclude that $\emi_{\text{0-1}}$ is as well.
One can then deduce that $\emd_{\text{0-1}}$ and $\emdhat_{\text{0-1}}$ give measures of the set difference between $A$ and $B$.
Specifically,  $\emd_{\text{0-1}}$ gives the set difference of the larger set from the smaller, and $\emdhat_{\text{0-1}}$ gives the set difference of the smaller set from the larger. 
The sum of both yields the symmetric difference.
One may also apply \refreq{pdpT} with $K=\emi_{\text{0-1}}$ or \refreq{cndpT} with $D=\emdhat_{\text{0-1}}$ and $p=\emptyset$ to obtain the Jaccard index and distance.

Switching to a ground distance other than the discrete metric is like allowing a degree of uncertainty in element identity.
The sharper or more concave the comparison function, the closer \ac{emd} and its derivatives are to their respective binary set operations.
The point $p$ is used to determine the cost of an unmatched element, which could potentially vary if some point is considered more important than another.
Practically, thresholding a ground distance by some upper bound can be used to artificially induce concavity and make comparisons more strict.

Another result that can be derived as a special case of \ac{emi} follows.

\begin{proposition}
If there exists a function $g: X \to \mathbb{R}$ such that the ground distance $D(x,y) = g(x)+g(y)$, then EMI$_p = 0$ and is trivially \posdef{} on $\mathcal{F}(X)$ for any choice of $p$.
\label{thm:constantEMD}
\end{proposition}
\begin{proof}
Let $f(a,b)$ be the maximum-cost maximum flow between sets $A$ and $B$ with respect to $K_p$ defined using Lemma~\ref{thm:defK} with $x_0=p$.
Note that in this case,
$
K_p(x,y) = 0.
$
As a result,
$
\emi_p(A,B)=0,
$
which is trivially \posdef{}.
\end{proof}
If $g(p) \geq 0$ and we opt to use \ac{emi}$^\prime$ by discarding $D(p,p)$ in \refreq{emdnot}, then 
\begin{equation}
\emi_p^\prime(A,B)=2g(p)\min\left\{\mu(A),\mu(B)\right\},
\end{equation}
which is simply a scaled version of the \posdef{} min-kernel.

We expect there to be many other instances of \posdef{} kernels based either directly or indirectly on \ac{emd}.
For example, Cuturi~\mycitet{c-sdlcot} proposed a regularized version of \ac{emd} via an additional entropic term that yields the exponent of the \textit{\ac{idk}},
\begin{equation}
\small
\mathit{IDK}(A,B) = \exp{\left(-u\sum_{a \in A}\sum_{b\in B}\dens_A(a)\dens_B(b)D(a,b)\right)},
\label{eq:idk}
\end{equation}
when the entropic term's effect is maximized.
Refer to Appendices~\ref{sec:transportationRealLine},~\ref{sec:transportationCircle},~\ref{sec:transportationL2Hypersphere} for further examples.
One should note that the computation of \ac{emi} (or \EMDHat{}) relies entirely on algorithms for computing \ac{emd}, so there is no significant difference computationally between the two.

\section{Experiments}
In this section we describe experiments with classification using \acp{svm} designed to demonstrate the utility of \EMDHat{} as well as the utility of the definite-preserving transformation of Section~\ref{sec:pdpT} with respect to \ac{emd}.
To our knowledge, \EMDHat{} (\textit{not} EMD) has not been applied in a kernel setting and we therefore perform the first such experiments.
In particular, we evaluate the effect of choosing some different values of $p$ (the sink to which excess mass is transported in our generalization of Pele and Werman's \EMDHat{}).
For each of the \ac{emd} variants, we make use of Theorem~\ref{thm:expK} to construct generalized \ac{rbf} kernels of the form $\exp{(-uD_{\emd})}$, where $D_{\emd}$ is an \ac{emd}-based distance between sets. 
In order to avoid the overhead of tuning $u$ via cross-validation, we assign $u$ to be the inverse of the average value of $D_{\emd}$ on the training set as suggested by Zhang et al.\@~\mycitet{zmls-lfkctoccs}.

We also show that when using unnormalized sets, especially when the magnitude of the mass has semantic significance relevant to classification, that \EMDHat{} is superior to \ac{emd}.
Since we are dealing with indefinite kernels, we evaluate the results in the context of two techniques designed to address the nonconvex optimization encountered in training \acp{svm} with such kernels.
The techniques mentioned are eigenvalue shifting of the kernel matrix and the \ac{ksvm} recently proposed by~\mycite{lco-lsvmks}.
Both methods were chosen for their relative simplicity of implementation as well as the fact that test points (or associated kernel evaluations) do not need to be modified. 
Where appropriate, these methods are balanced against \acp{svm} trained directly with the indefinite kernels (see Tables~\ref{tab:mnist} and~\ref{tab:posture}).

\textit{Shift} is a heuristic that involves shifting the eigenvalues of the kernel matrix to be non-negative (\eg by adding $s\mathbf{I}$ to the kernel matrix, where $s$ is the amount to shift each eigenvalue and $\mathbf{I}$ is the identity matrix).
Shifting causes the \ac{svm} training problem to become convex, assuring a globally optimal solution.
Wu et al.\@~\mycitet{wcz-atnpsdsmkm} show that shifting adds a regularization term that penalizes the norm of the support vector coefficients.
Thus, simply choosing a very large $s$ that guarantees \posdef{}-ness is not necessarily beneficial as it may constrain possible solutions.
The smallest possible $s$ (\ie the magnitude of the least negative eigenvalue) is generally a good default choice.
Approximations for $s$ that assure \posdef{}-ness without requiring an eigendecomposition of the kernel matrix can be used.
We did not make use of these approximations, however. 

On the other hand, \ac{ksvm} is formulated in the theory of Krein spaces (generalizations of Hilbert spaces with indefinite inner products) and may be considered a state-of-the-art indefinite kernel technique.
Our results certainly reflect its ability to compensate for deficiencies in an indefinite kernel.
However, \ac{ksvm} is computationally expensive, requiring an eigendecomposition of the entire precomputed kernel matrix used for training.
Therefore, Loosli et al.\@~\mycitet{lco-lsvmks} also proposed \ac{ksvm}-L, a more practical alternative that uses partial decompositions.

For completeness, we briefly describe the \ac{ksvm} algorithm.
Given a kernel matrix $G_K$ and label vector $\vec{y}$ containing $\pm 1$ for each respective positive or negative instance, one must compute an eigendecomposition of $YG_KY$, where $Y = \diag{\vec{y}}$ is an otherwise zero matrix with $\vec{y}$ on the diagonal.
If $U$ and $D$ are the resulting eigenvector and eigenvalue matrices satisfying $UDU^\intercal=YG_KY$, then one trains the \ac{svm} using a standard solver with the \posdef{} kernel matrix $\overline{G}_K=USDU^{\intercal}$, where $S = \sign{D}$.
Finally, one transforms the resulting support vector coefficients $\vec{\overline{\alpha}}$ (not to be confused with $\alpha$ in \EMDHat[\alpha]) to obtain support vector coefficients $\vec{\alpha} = USU^\intercal\vec{\overline{\alpha}}$ in the original indefinite space.
The solution is not sparse.
\ifexpanded
One may note that \ac{ksvm} is equivalent to flipping each negative eigenvalue of the kernel matrix to be positive prior to transforming the result.
\fi
We also note that a one-versus-all scheme for multiclass \acp{svm} has a distinct computational advantage over one-versus-one schemes since if $\vec{y}_i$ is the label vector treating the $i$-th class as positive and the remainder negative and $V$ contains the eigenvectors of $G_k$, then $U_i=Y_iV$ provides the eigenvectors of $Y_iG_kY_i$.
Consequently, only one eigendecomposition is required regardless of the number of classes.
\ifexpanded
We take advantage of this fact in our experiments; \ie all results are computed using one-versus-all binary \acp{svm}.
\else
We take advantage of this fact by using one-versus-all binary \acp{svm} in all of our experiments.
\fi

For select datasets (see Section~\ref{sec:doex}), we compare against more traditional kernels including linear, Gaussian, \ac{pmk}~\mycite{gd-pmkelsf}, and \ac{idk}~\cite{c-sdlcot}.
Linear and Gaussian kernels require special treatment as there is no obvious, uncontrived way for them to handle unordered sets of features.
We lexicographically sorted the feature vectors of an instance and concatenated them prior to using each of the two former kernels, appending zeros for smaller instance sets.
\ac{pmk} consists of constructing a multi-resolution $d$-dimensional histogram with resolutions ranging from $1$ to $L^d$ bins.
\ac{idk} is given by~\refreq{idk}, and we use the same selection criteria for $u$ here as we do for the other \ac{emd}-based kernels.
The former three kernels are \posdef{}, whereas \ac{idk} is indefinite on unnormalized sets/histograms.
We thus apply Shift and Krein SVM to \ac{idk}.

\subsection{Datasets}
Each considered kernel---\ac{emd} with Rubner's scaling, \EMDHat{}, and its biotope transformation \EMDHat{}$_{T,p}$ (hereafter referred to as \ac{emjd})---was evaluated on four datasets: the texture database KTH-TIPS~\mycite{hcfe-osrwcmc}, the object category database Caltech-101~\mycite{ffp-osloc}, a dataset Corner-MNIST (CMNIST) based on the handwritten character database MNIST~\mycite{mnist}, and a motion capture hand posture dataset collected by the authors\footnote{The posture dataset is available at \url{http://www2.latech.edu/~jkanno/collaborative.htm}.}.
The Euclidean distance served as the ground distance for each dataset except for Caltech-101, for which it was squared.

The KTH-TIPS database consists of 10 texture classes under varying scale, pose, and illumination with 81 instances per class.
Images are standardized by resizing to a horizontal resolution of 480 pixels while preserving aspect ratio.
We adopted much of the experimental design of Zhang et al.\@~\mycitet{zmls-lfkctoccs}, constructing image \textit{signatures} from SIFT descriptors. 
The SIFT descriptor~\mycite{l-difsik} computes an $N$-bin histogram of image gradient orientations for an $M \times M$ grid of samples in the region of interest, resulting in an $M\times M\times N$ dimensional vector.
We used the implementation of the SIFT descriptor provided by Vedaldi and Fulkerson~\mycitet{vf-vlfeat} with $M=4$ and $N=8$.
The resulting 128-dimensional vectors were scaled to have a Euclidean norm of 1 to reduce the influence of illumination changes.
The descriptors were then clustered using a $k$-means algorithm (with $k=40$).
Each mean was weighted with the percentage of descriptors assigned to it, and the means paired with these weights constituted the so-called signature for a single image.

A very similar feature extraction procedure was conducted for the Caltech-101 dataset composed of color images of 101 categories (\eg face, car, etc.) with varied presentation. 
Instead of SIFT descriptors, the PHOW descriptor implemented by Vedaldi and Fulkerson~\mycite{vf-vlfeat} was used to represent images.
At a high level, the PHOW descriptor is a dense SIFT extractor (the regions of interest are densely sampled in a grid) that can operate on multiple color channels instead of just grayscale.
However, we simply used grayscale.
Sets were normalized for both KTH-TIPS and Caltech-101.

The MNIST dataset comprises $28\times28$ grayscale images of handwritten digits ranging from 0 to 9.
CMNIST was created in the following manner.
Noble's version~\mycite{n-fc} of the Harris corner detector~\mycite{hs-cced} was used to identify keypoints in each image (implemention again provided by Vedaldi and Fulkerson~\mycitet{vf-vlfeat}).
Images were smoothed with a Gaussian window with a variance of 1 prior to application of the Harris response function, which also used a Gaussian window with a variance of 1. 
Local maxima in the response were interpreted as corners.
The set of coordinates (scaled to lie between 0 and 1) of these detected corners then constitute the features of the image with the expected number of corners and their locations depending upon the digit.
The number of detected corners typically ranged from 5 to 15.

The final dataset consists of variable size unlabeled 3D point set examples of 5 hand postures captured in a Vicon motion capture environment for 12 users.
Each unlabeled point represents the 3D position of a motion capture marker attached to a glove in a coordinate system localized to the user's hand (via a rigid pattern of labeled markers affixed to the back of the glove). 
Eleven unlabeled markers were attached to the glove.
However, not all markers are necessarily detected at a given time due to occlusion and the relative size of the markers with respect to the capture space.
Therefore, the number of unlabeled points in an instance ranges from 3 to 12 (extraneous markers, though uncommon, are also possible).
Markers more than 200 millimeters from the local  origin were pruned, and sets with fewer than 3 markers were discarded.
No other preprocessing or feature extraction was necessary; points remained at millimeter scale.

\subsection{Design of Experiments}\label{sec:doex}
Each experiment on each dataset involves the choice of a different sink $p$ to which excess mass is sent.
If the ground distance is thresholded and $p$ lies beyond the threshold for every point in the training and test sets, then one can use a flat rate equal to the threshold as the cost of transporting excess mass.
Therefore, we simply use the threshold to identify different experiments.
The thresholds used are reported in the provided tables (\eg Table~\ref{tab:texture}).
One will note that the bottom row of each table has no threshold (denoted by a dash), and in this case $p$ was generally chosen to be the origin with the exception of CMNIST, where it was chosen to be the center of an image, $\begin{bmatrix}0.5, 0.5\end{bmatrix}^\intercal$.
In the case of KTH-TIPS and Caltech-101, choosing the origin is not much different than choosing a threshold of 1 since every point lies on the surface of a unit hypersphere.
The advantage of flat thresholds lies in their simplicity of implementation (the precise value of the optimal flow is irrelevant) as well as the ability to use certain faster algorithms~\mycite{pw-fremd}.
In addition, \ac{idk} is compatible with thresholded ground distances, but the concept of a sink $p$ does not apply (hence the N/A entries in Tables~\ref{tab:mnist} and~\ref{tab:posture}).
Linear, Gaussian, and \ac{pmk} kernels, on the other hand, do not use a ground distance at all, and thus their results are presented separately in Table~\ref{tab:traditional}.
We only evaluated these alternative kernels on CMNIST and the posture recognition dataset.
For \ac{pmk}, we chose $L=128$ for CMNIST and $L=256$ for posture recognition.

The following data selection schemes were repeated for each experiment (threshold) with the exception that the selection of data for experiments with no threshold matched that of the highest threshold in order to enable a direct comparison.
For KTH-TIPS (and Caltech-101), 40 (15) images from each class were randomly drawn to be the training set with an equivalently drawn disjoint test set. 
This random selection was repeated 5 times in order to obtain 5 training/test set pairs, the results of which were averaged.
For CMNIST, 200 examples from each class were randomly chosen and 5-fold cross validation was computed for each experiment.
For the posture recognition dataset, special consideration was required due to the fact that there is signficant correlation and even near duplication for samples corresponding to a single user.
Therefore, a leave-one-user-out approach was employed where each of the 12 users served in turn as the test set. 
As a result, experiments measured the generalization of the classifier to new users.
The dataset's size was reduced and classes balanced by randomly selecting 75 examples per class per user.

\begin{table*}
\centering
\minipage{0.65\textwidth}
\centering
\caption{Accuracies for texture and object category recognition on normalized sets with KTH-TIPS and Caltech-101. All kernels were found to be positive definite. Since sets are normalized, \ac{emd} is equal to \EMDHat{}.}
\begin{tabular}{| c ||c c | c c |}
\hlx{hv[2-3,4-5]} 
&\multicolumn{2}{c|}{\bf KTH-TIPS}& \multicolumn{2}{|c|}{\bf Caltech-101} \\
\hlx{hv}
Threshold & $\emd/\emdhat$ & $\mathit{EMJD}$ & $\emd/\emdhat$ & $\mathit{EMJD}$   \\
\hlx{vhv}
0.5 &  $\underset{\pm 6.19}{71.45}$ & $\underset{\pm 6.16}{70.95}$ & $\mathbf{\underset{\pm 0.90}{49.97}}$ & $\underset{\pm 0.80}{49.65}$   \\
1 & $\mathbf{\underset{\pm 1.00}{74.75}}$ & $\underset{\pm 0.65}{74.55}$ & $\underset{\pm 0.75}{48.77}$ & $\underset{\pm 0.81}{48.84}$  \\
$\sqrt{2}$ & $\underset{\pm 7.96}{70.70}$ & $\underset{\pm 8.06}{70.85}$ & $\underset{\pm 1.39}{48.57}$ & $\underset{\pm 1.19}{48.71}$ \\
- & $\underset{\pm 7.96}{70.70}$ & $\underset{\pm 8.07}{70.80}$ & $\underset{\pm 1.39}{48.57}$ & $\underset{\pm 1.26}{48.55}$ \\
\hlx{h}
\end{tabular}
\label{tab:texture}
\endminipage
\minipage{0.35\textwidth}
\centering
\caption{Accuracies for ground-distance-invariant \posdef{} kernels evaluated on CMNIST and the posture recognition dataset.}
\begin{tabular}{| c || c c |}
\hlx{hv}
Kernel & \textbf{CMNIST} & \textbf{Posture} \\
\hlx{vhv}
Linear & $\underset{\pm 1.70}{17.15}$ & $\underset{\pm 12.92}{71.71}$ \\
Gaussian & $\underset{\pm 2.65}{42.00}$ & $\underset{\pm 9.77}{81.09}$ \\
\ac{pmk} & $\underset{\pm 1.72}{47.90}$ & $\underset{\pm 13.35}{69.42}$ \\
\hlx{h}
\end{tabular}
\label{tab:traditional}
\endminipage

\vskip 5pt
\caption{Accuracies for handwritten character recognition on unnormalized sets with the CMNIST data.}
\begin{tabular}{| c ||c c c c | c  c c c | c  c c c|}
\hlx{hv[2-5,6-9, 9-12]} 
&\multicolumn{4}{c|}{\bf Indefinite}& \multicolumn{4}{|c|}{\bf Shift}&  \multicolumn{4}{c|}{\bf KSVM} \\
\hlx{hv}
Threshold & $\emd$ & $\emdhat$ & $\mathit{EMJD}$ & \acs{idk} & $\emd$ & $\emdhat$ & $\mathit{EMJD}$ & \acs{idk} &  $\emd$ & $\emdhat$ & $\mathit{EMJD}$  & \acs{idk}\\
\hlx{vhv}
0.25 & $\underset{\pm 5.45}{34.30}$ & $\underset{\pm 1.36}{67.80}$ & $\mathbf{\underset{\pm 2.22}{78.20}}$ & $\underset{\pm  2.01}{19.75}$  & $\underset{\pm 2.36}{32.25}$ & $\underset{\pm 1.76}{79.90}$ & $\mathbf{\underset{\pm 1.97}{80.65}}$ & $\underset{\pm 0.99}{32.55}$ & $\underset{\pm 1.78}{75.30}$ & $\underset{\pm 1.56}{78.05}$ & $\mathbf{\underset{\pm 1.99}{79.50}}$ &$\underset{\pm 1.10}{76.85}$\\
0.5 & $\underset{\pm 4.03}{28.10}$ & $\underset{\pm 2.77}{60.30}$ & $\underset{\pm 3.22}{73.90}$ & $\underset{\pm 2.41 }{ 16.00}$  & $\underset{\pm 1.22}{28.70}$ & $\underset{\pm 1.90}{78.80}$ & $\underset{\pm 1.74}{78.85}$ & $\underset{\pm  0.73}{25.40}$ & $\underset{\pm 1.34}{75.30}$ & $\underset{\pm 0.98}{76.00}$ & $\underset{\pm 0.84}{76.90}$ &$\underset{\pm 1.97}{74.35}$\\
1 & $\underset{\pm 3.43}{32.70}$ & $\underset{\pm 0.38}{58.65}$ & $\underset{\pm 1.11}{67.10}$ & $\underset{\pm  0.72}{ 16.35}$  & $\underset{\pm 2.37}{29.10}$ & $\underset{\pm 2.19}{77.70}$ & $\underset{\pm 1.93}{77.45}$ & $\underset{\pm  1.87}{24.45}$  & $\underset{\pm 1.81}{72.15}$ & $\underset{\pm 1.62}{73.65}$ & $\underset{\pm 1.61}{74.85}$ &$\underset{\pm 1.53}{75.95}$\\
$\sqrt{2}$ & $\underset{\pm 2.71}{32.75}$ & $\underset{\pm 0.72}{59.90}$ & $\underset{\pm 1.81}{65.45}$ & $\underset{\pm  1.52}{ 16.80}$  & $\underset{\pm 1.46}{27.85}$ & $\underset{\pm 2.03}{77.70}$ & $\underset{\pm 1.85}{77.75}$ & $\underset{\pm  0.75}{ 23.85}$  & $\underset{\pm 2.00}{76.05}$ & $\underset{\pm 1.23}{74.70}$ & $\underset{\pm 1.72}{74.65}$ &$\underset{\pm 0.76}{77.25}$\\
- & $\underset{\pm 2.71}{32.75}$ & $\underset{\pm 1.56}{49.60}$ & $\underset{\pm 2.05}{52.00}$ & N/A  & $\underset{\pm 1.46}{27.85}$ & $\underset{\pm 1.67}{75.30}$ & $\underset{\pm 1.93}{76.85}$ & N/A  & $\underset{\pm 2.00}{76.05}$ & $\underset{\pm 1.11}{73.85}$ & $\underset{\pm 1.01}{74.35}$ &  N/A  \\
\hlx{h}
\end{tabular}
\label{tab:mnist}

\vskip 5pt
\caption{Accuracies for posture recognition on unnormalized sets.}
\begin{tabular}{| c ||c c c c | c  c c c | c  c c c|}
\hlx{hv[2-5,6-9, 9-12]} 
&\multicolumn{4}{c|}{\bf Indefinite}& \multicolumn{4}{|c|}{\bf Shift}&  \multicolumn{4}{c|}{\bf KSVM} \\
\hlx{hv}
Threshold & $\emd$ & $\emdhat$ & $\mathit{EMJD}$ & \acs{idk} & $\emd$ & $\emdhat$ & $\mathit{EMJD}$ & \acs{idk} &  $\emd$ & $\emdhat$ & $\mathit{EMJD}$  & \acs{idk}\\
\hlx{vhv}
25 & $\underset{\pm 16.56}{37.20}$ & $\underset{\pm 11.11}{80.87}$ & $\underset{\pm 10.53}{80.53}$  & $\underset{\pm  11.79}{23.82 }$ & $\underset{\pm 15.42}{53.31}$ & $\underset{\pm 11.15}{80.64}$ & $\underset{\pm 10.53}{80.53}$  & $\underset{\pm  11.23}{ 33.69}$  & $\underset{\pm 13.76}{73.00}$ & $\underset{\pm 10.99}{80.67}$ & $\underset{\pm 10.53}{80.53}$  & $\underset{\pm  10.90}{ 69.13}$ \\
50 & $\underset{\pm 18.65}{38.96}$ & $\underset{\pm 12.03}{90.91}$ & $\underset{\pm 12.00}{90.96}$  & $\underset{\pm  12.81}{30.11 }$  & $\underset{\pm 17.87}{42.20}$ & $\underset{\pm 11.76}{91.13}$ & $\underset{\pm 12.00}{90.96}$  & $\underset{\pm 11.57}{ 41.33}$  & $\underset{\pm 13.36}{87.98}$ & $\underset{\pm 12.06}{90.96}$ & $\underset{\pm 12.00}{90.96}$   & $\underset{\pm  11.31}{ 72.69}$  \\
100 & $\underset{\pm 20.22}{32.80}$ & $\underset{\pm 6.37}{95.02}$ & $\underset{\pm 6.63}{94.44}$  & $\underset{\pm  14.02}{ 34.47}$  & $\underset{\pm 16.94}{34.07}$ & $\underset{\pm 6.40}{95.00}$ & $\underset{\pm 6.63}{94.44}$  & $\underset{\pm  13.26}{ 43.56}$  & $\underset{\pm 10.06}{92.93}$ & $\underset{\pm 6.12}{95.00}$ & $\underset{\pm 6.63}{94.44}$  & $\underset{\pm  12.39}{ 73.13}$ \\
150& $\underset{\pm 22.31}{28.96}$ & $\mathbf{\underset{\pm 6.40}{95.47}}$ & $\underset{\pm 6.60}{95.02}$  & $\underset{\pm  13.89}{ 32.78}$  & $\underset{\pm 16.30}{30.69}$ & $\underset{\pm 6.77}{95.00}$ & $\underset{\pm 6.60}{95.02}$  & $\underset{\pm  13.96}{ 39.18}$  & $\underset{\pm 11.92}{91.82}$ & $\underset{\pm 6.54}{95.42}$ & $\underset{\pm 6.60}{95.02}$  & $\underset{\pm  11.60}{ 72.98}$ \\
200 & $\underset{\pm 18.65}{29.73}$ & $\underset{\pm 6.73}{95.09}$ & $\underset{\pm 7.17}{94.31}$  & $\underset{\pm  14.27}{ 36.07}$  & $\underset{\pm 16.82}{30.89}$ & $\underset{\pm 7.20}{94.44}$ & $\underset{\pm 7.22}{94.24}$  & $\underset{\pm  12.93}{ 40.91}$  & $\underset{\pm 8.43}{92.22}$ & $\underset{\pm 7.22}{94.60}$ & $\underset{\pm 7.23}{94.27}$  & $\underset{\pm  11.66}{ 73.09}$ \\
- & $\underset{\pm 18.65}{29.73}$ & $\underset{\pm 5.97}{95.20}$ & $\underset{\pm 6.07}{95.24}$ & N/A & $\underset{\pm 16.82}{30.89}$ & $\mathbf{\underset{\pm 5.69}{95.27}}$ & $\underset{\pm 6.15}{95.09}$  & N/A  & $\underset{\pm 8.43}{92.22}$ & $\mathbf{\underset{\pm 5.77}{95.60}}$ & $\underset{\pm 5.92}{95.58}$  & N/A \\
\hlx{h}
\end{tabular}
\label{tab:posture}
\end{table*}

\subsection{Results and Discussion}
For normalized sets contained in KTH-TIPS and Caltech-101 (Table~\ref{tab:texture}), there is no significant difference between the three kernels. 
The main point we make with these two datasets is that there is no degradation in performance when using \EMDHat{} or \ac{emjd} versus \ac{emd}.
In fact, \ac{emd} and \EMDHat{} are the exact same for any two normalized sets since the difference in mass is zero.
However, for unnormalized sets (Tables~\ref{tab:mnist} and~\ref{tab:posture}), \EMDHat{} and \ac{emjd} are noticeably better than \ac{emd} despite the indefinite kernel techniques.
\ac{ksvm} actually improved \ac{emd}'s accuracy far beyond what was expected, nearly matching \EMDHat{}'s performance (and surpassing it on the highest thresholds for CMNIST).
However, this state-of-the-art indefinite kernel technique was still unable to bridge the difference in all cases, and the results should be balanced by the more computationally practical Shift, which was completely unable to compensate for \ac{emd}'s indefiniteness.

Our experiments on KTH-TIPS and Caltech-101 confirmed the report of Zhang et al.\@~\mycitet{zmls-lfkctoccs} that the \ac{rbf} kernel for \ac{emd} is \posdef{} with this data.
However, computation of \ac{emi} revealed an indefinite kernel matrix, which indicates that only a subset of $u < 0$ from Theorem~\ref{thm:expK} is satisfied and that  Zhang et al.'s selection strategy for $u$ just happens to fall within this subset.
The same behavior was observed for \ac{emjd} on these two datasets.
The ground distance's support for posture recognition and CMNIST, on the other hand, does not consist of normalized vectors.
For posture recognition, we noticed that \ac{emjd} was more likely to yield a \posdef{} \ac{rbf} using the aforementioned selection strategy.
For example, observe that the Shift and \ac{ksvm} results are the same as the indefinite results for certain thresholds, with lower threshelds apparently increasing the likelihood of generating a \posdef{} kernel.
Exploration on normalized sets (not shown) with both CMNIST and posture recognition made this effect more pronounced.

Of special note is the fact that \EMDHat{} and \ac{emjd} yield significant improvements in accuracy even without applying any indefinite kernel technique.
On the posture recognition dataset in particular, the effective results are nearly indistinguishable from Shift and \ac{ksvm}.
For the CMNIST dataset, indefinite \ac{emjd} consistently outperformed the other two kernels and rivaled Shift and \ac{ksvm} at the lowest threshold.
These results indicate that \EMDHat{}, \ac{emjd}, and perhaps the definite-preserving transformation in general have value on their own without additional indefinite kernel methods.

A comparison against the more traditional kernels (Table~\ref{tab:traditional} and the \acs{idk} columns of Tables~\ref{tab:mnist} and~\ref{tab:posture}) only reinforces this conclusion.
Though \ac{idk} performed comparably well on CMNIST with Krein SVM, it performed considerably worse in all other cases and appeared to be insensitive to changes in the threshold.
The \posdef{} kernels performed relatively poorly as well, although the linear and Gaussian kernels performed surprisingly well for posture recognition compared to both \ac{pmk} and \ac{idk}.

In general, one can observe that the threshold has a significant effect on the quality of the classifier.
The highest thresholds, which matched or exceeded the diameter of each dataset's support, did not yield the best observed results for any dataset.
Lower thresholds tended to yield better results (up to a point). 
As the threshold lowers, \ac{emd} becomes a closer approximation to the set symmetric difference and thus more similar to the intersection kernel.
As stated in Section~\ref{sec:emi}, thresholding can be interpreted as a means to induce concavity in the ground distance and make it more similar to the discrete metric.
This explains why the accuracy drops off after a certain minimum threshold (as it becomes too similar to classical intersection to associate slightly different elements) as well as its tendency to improve prior to the drop off.

\section{Conclusion}
In this paper we presented proof that \posdef{} kernels can be derived via Lemma~\ref{thm:defK} from \ac{emd} and are dependent on the ground distance and the space in which it operates.
We set our discussions in the context of set theory, providing motivation for our derivations and an intuitive interpretation of \ac{emd}'s value, namely as a generalization of otherwise binary set operations.
In doing so, we generalized \EMDHat{} for kernels.
We also proposed a \posdef{}-preserving transformation that normalizes a kernel's values and showed that the Jaccard index is simply the transformation of the intersection kernel.
As a corollary, the biotope transform was shown to preserve \cnegdef{} \ifexpanded as well asmetric \fi  properties.
Finally, we provided the first assessment of \EMDHat{} in a kernel setting and showed that it and its biotope transform \ac{emjd} achieve superior accuracy over \ac{emd} on experiments with unnormalized sets and a state-of-the-art indefinite kernel technique.
Indeed, we showed that an indefinite kernel technique may not even be necessary.
\ac{emjd} was found to have more favorable numerical properties than \EMDHat{}.

Our work raised some open questions.
We do not know whether thresholding a distance preserves \cnegdef{} properties as it does metric properties~\mycite{pw-fremd}.
Our experiments did not contradict the hypothesis.
The choice of the optimal threshold is also open.
One could always tune the threshold via cross-validation, but we suspect that a decent approximation to the optimal threshold would be to use the average or median distance between all points.
Using no threshold or choosing $p$ to be closer than the threshold is also an option to consider as the posture recognition experiments demonstrate.

One unexpected result was \ac{ksvm}'s poor performance on CMNIST relative to Shift for \EMDHat{} and \ac{emjd}.
This result is at odds with the expectation that \ac{ksvm} should be at least as good as other indefinite kernel techniques, which is fairly well justified in its introductory article~\mycite{lco-lsvmks}.
We noted that the eigenspectrum of a CMNIST kernel matrix was much less concentrated than those for the other datasets.
Whereas performing a partial decomposition with the 50 highest magnitude eigenvalues was typically sufficient to retain approximately $95\%$ of the spectrum's total magnitude on the other datasets, as many as $1200$ eigenvalues were required to achieve the same \ifexpanded preservation of the spectrum \else retention \fi on CMNIST.
In fact, the results reported in Table~\ref{tab:mnist} are from a complete decomposition.
Additional research may be required to determine if this is due to a peculiarity unique to CMNIST or some property of \ac{ksvm}.


%

\appendices
\section{Nested Transformations}\label{sec:nestedT}
This Appendix gives proofs of Propositions~\ref{thm:pdpTLimits} and~\ref{thm:inducedPDPT} and introduces Proposition~\ref{thm:closedFormKTN}.
In Appendix~\ref{sec:transportationL2Hypersphere}, we use Proposition~\ref{thm:inducedPDPT} to find ground distances for which \ac{emd} is \cnegdef{}.
\begin{proof}[Proof of Proposition~\ref{thm:pdpTLimits}]
By~\refreq{pdpT} and~\refreq{fixedPointKT}, we trivially see that if $x=y$, $\ptransform{K}{n}(x,y)=1$ for all $n$.
If $x \ne y$, , we deduce from~\refreq{posDefTriangle} that $\ptransform{K}{}(x,y)< 1$ and 
\begin{equation}
\frac{\ptransform{K}{n+1}(x,y)}{\ptransform{K}{n}(x,y)}=\frac{1}{2-\ptransform{K}{n}(x,y)}< 1,
\end{equation}
\begin{equation}
\frac{\ptransform{K}{n+1}(x,y)}{\ptransform{K}{n}(x,y)} < \begin{cases}
\frac{1}{2} & \text{if } \ptransform{K}{n}(x,y) < 0,\\
 \frac{1}{2-\ptransform{K}{}(x,y)} & \text{if } \ptransform{K}{n}(x,y) > 0.
\end{cases}
\end{equation}
Noting that the geometric sequence $a_n=|\ptransform{K}{}(x,y)/r^{n-1}|$, where ${r=2-\ptransform{K}{}(x,y)}$ if $\ptransform{K}{}(x,y) >0$ and $r=2$ otherwise, converges to 0 as $n\to\infty$, we conclude that the sequence $b_n=|\ptransform{K}{n}(x,y)|$, which is bounded below by 0 and above by $a_n$ for $n \geq 1$, also converges to 0 by the squeeze theorem~\mycite{s-cet}.
The proposition follows.
\end{proof}

\begin{proposition}
For any given $K:X\times X \to \mathbb{R}$,
\begin{equation}
\small
\ptransform{K}{n}(x,y) = \frac{K(x,y)}{2^{n-1}[K(x,x)+K(y,y)]-(2^n-1)K(x,y)}.
\end{equation}
\label{thm:closedFormKTN}
\end{proposition}
\begin{proof}
We prove this via induction.
As a base case, note that $n=1$ yields~\refreq{pdpT}.
Now suppose that the proposition holds for some $n-1$, $n > 1$.
Then by~\refreq{fixedPointKT},
\begin{align}
\small
&\ptransform{K}{n}(x,y) = \frac{\ptransform{K}{n-1}(x,y)}{2-\ptransform{K}{n-1}(x,y)}\\
&=\frac{\frac{K(x,y)}{2^{n-2}[K(x,x)+K(y,y)]-(2^{n-1}-1)K(x,y)}}{2-\frac{K(x,y)}{2^{n-2}[K(x,x)+K(y,y)]-(2^{n-1}-1)K(x,y)}}\\
&=\frac{K(x,y)}{2^{n-1}[K(x,x)+K(y,y)]-(2^{n}-2)K(x,y)-K(x,y)}\\
&=\frac{K(x,y)}{2^{n-1}[K(x,x)+K(y,y)]-(2^{n}-1)K(x,y)}.
\end{align}
\end{proof}
Note that although we focused on $n\geq1$, $n$ could be considered a continuous hyperparameter within the range $(-\infty, \infty)$.
If $n=0$, then we obtain a generalization of the F measure (as interpreted as a kernel by Ralaivola et al.\@~\mycitet{rssb-gkci})
\begin{equation}
\ptransform{K}{0}(x,y) = \frac{2K(x,y)}{K(x,x)+K(y,y)}.
\end{equation}
\posdef{}-ness is not guaranteed for $n < 0$.

\begin{proof}[Proof of Proposition~\ref{thm:inducedPDPT}]
Consider the kernel matrix $G_K^{(n)}=[\ptransform{K}{n}(x_i,x_j)]$ for some selection of elements $x_1,\dots,x_n \in X$ with $1 \leq i,j \leq n$.
Since the definition of a \posdef{} kernel requires only distinct elements for \refreq{defKEq}, we may without loss of generality assume that each element is distinct, \ie $i \ne j \implies x_i \ne x_j$.
We show that $\lim_{n\to\infty}G_K^{(n)} = \mathbf{I}$, and hence its eigenvalues each converge to 1, implying that the minimum eigenvalue must be greater than or equal to 0 for some $n_0$ and thus $G_K^{(n)}$ is \posdef{} for $n \geq n_0$.

We show this by proving a stronger version of~\refreq{pdpTLimits}.
Note that since the considered kernel $K$ is not necessarily \posdef{}, $\ptransform{K}{}$ is not necessarily in the interval $[-1/3, 1]$.
However, $\ptransform{K}{}(x,x) = 1$ for any $x \in X$.
More importantly, the derivation of~\refreq{pdpTLimits} holds if $\ptransform{K}{}(x,y) < 1$, regardless of the magnitude $|\ptransform{K}{}(x,y)|$.
Therefore, it is sufficient to show that $\ptransform{K}{n}(x,y)$ eventually becomes negative for any $x,y$ such that $\ptransform{K}{}(x,y) > 1$.
We consider three cases: 
\begin{enumerate}
\item $\ptransform{K}{}(x,y) > 2$. 

Note $\ptransform{K}{2}(x,y) < 0$ since the denominator of~\refreq{fixedPointKT} after substitution is negative but the numerator is positive.
\item $\ptransform{K}{}(x,y) = 2$.

Apply Proposition~\ref{thm:closedFormKTN} to $\ptransform{K}{}$ with $n=2$ to yield $\ptransform{\ptransform{K}{}}{2}(x,y) = \ptransform{K}{3}(x,y) = -1.$
\item $\ptransform{K}{}(x,y) \in  (1, 2)$.

Note that $\ptransform{K}{2}(x,y) > \ptransform{K}{}(x,y) $ since the denominator of~\refreq{fixedPointKT} after subtitution will be less than one.
Furthermore, the denominator decreases with each iteration, which implies that $\ptransform{K}{n}(x,y)$ will increase until it falls within one of the previous two categories.
\end{enumerate}
Thus $\lim_{n\to \infty}\ptransform{K}{n}(x_i,x_j) = 0$ for $x_i \ne x_j$ and $\lim_{n\to\infty}G_K^{(n)} = \mathbf{I}$.
\end{proof}


\section{Transportation on the Real Line}\label{sec:transportationRealLine}
Consider the space of probability distributions on the real line $\mathbb{R}$.
Let $D: \mathbb{R} \times \mathbb{R} \to \mathbb{R}_0^+$ be a convex, non-negative symmetric function that takes the form $D(a,b) = h(a-b)$, where $h:\mathbb{R}\to\mathbb{R}^+_0$.
If $D$ is \cnegdef{}, then one can show that \ac{emd} equipped with $D$ is \cnegdef{} as well.
A well known result~\mycite{rdg-tdc} states that \ac{emd} between two probability distributions $u, v \in \mathcal{P}(\mathbb{R})$ with a ground distance such as $D$ can be written
\begin{equation}
\emd(u,v) = \int_0^1 D(U^{-1}(s),V^{-1}(s))\dif s,
\end{equation}
where $U^{-1}$ and $V^{-1}$ are the inverse cumulative distribution functions of $u$ and $v$.
In essence, the $i$-th point in ascending order of one distribution maps to the $i$-th point of the other.
Since \ac{emd} in this form is clearly just the summation of \cnegdef{} functions, then \ac{emd} must also be \cnegdef{}.

Kolouri et al.\@~\mycitet{kzr-swkpd} use $D(a,b) = (a-b)^2$ to show that the \textit{sliced Wasserstein kernel}, which is calculated between distributions in $\mathbb{R}^d$ via one-dimensional projections, is \posdef{}.
One may also consider the following special case to reveal similarities to another min-like kernel, the Brownian bridge product kernel~\mycite{cfm-uhssvduibbk,szh-rkbsl1n},
\begin{equation}
K_B(x,y) = \min\{x,y\}-xy.
\end{equation}
Suppose the ground distance $D$ is supported by two points $p_1,p_2 \in \mathbb{R}$, and without loss of generality assume $p_1=-p_2=1$.
Assuming $u,v \in \mathcal{P}(\{p_1,p_2\})$, let $\chi_u(p_1) = x$ and $\chi_v(p_1) = y$ so that $\chi_u(p_2) = 1-x$ and $\chi_v(p_2) = 1-y$.
Then the optimal flow $f^*(p_i,p_i) = \min\{\chi_u(p_i),\chi_v(p_i)\}$ for $i \in \{1,2\}$, and $f^*(p_1,p_2)+f^*(p_2,p_1) = 1-f^*(p_1,p_1)-f^*(p_2,p_2)$.
Choosing the sink $p=0$ in \refreq{emi}, we can determine  that $\emi_0$ is the sum of two Brownian bridge product kernels and a similarly structured term:
\ifIEEE
\begin{equation}
\begin{split}
\emi_0(u, v) &= \sum_{i,j=1}^2 f^*(p_i,p_j)p_ip_j\\
& =   2(\min\{x,y\}+\min\{1-x,1-y\})-1\\
&= K_B(x,y) +2K_B(1-x,1-y)\\
&\quad{}+ \min\{x,y\} -x(1-y)\\
&\quad{}-y(1-x)+(1-x)(1-y).
\end{split}
\end{equation}
\else
\begin{equation}
\begin{split}
\emi_0(u, v) = \sum_{i,j=1}^2 f^*(p_i,p_j)p_ip_j & =   2(\min\{x,y\}+\min\{1-x,1-y\})-1\\
&= K_B(x,y) +2K_B(1-x,1-y)\\
&\quad{}+ \min\{x,y\} -x(1-y)\\
&\quad{}-y(1-x)+(1-x)(1-y).
\end{split}
\end{equation}
\fi

\section{Transportation on the Circle}\label{sec:transportationCircle}
Transportation on the circle is similar to transportation on the real line.
In fact, one simply has to find an optimal point at which to cut the circle prior to treating it like the real line.
In this case, the geodesic distance (\ie length of arc or angle) is used to compare points.
If the points $x,y$ are linearly indexed on $S^1$, the circle with radius 1, then
\ifexpanded
\begin{equation}
D(x,y) = \min\{|x-y|,2\pi-|x-y|\},
\end{equation}
or equivalently
\fi
\begin{equation}
D(x,y) = \arccos{\left(\begin{bmatrix}\cos(x) &\sin(x)\end{bmatrix}\begin{bmatrix}\cos(y)& \sin(y)\end{bmatrix}^\intercal\right)},
\end{equation}
which is provably \cnegdef{} by an infinite series expansion~\mycite{hr-akmlhdc}.
With the given ground distance and probability distributions $u, v \in \mathcal{P}(S^1)$, it can be shown that
\begin{equation}
\emd(u,v) = \|U-V-\alpha\|_1 = \int_0^{2\pi}|U(s)-V(s)-\alpha|\dif s
\label{eq:transportCircle}
\end{equation}
where $U$ and $V$ are cumulative distribution functions and $\alpha$ is the weighted median of $U-V$~\mycite{dss-ftomcc,rdg-tdc}.
Surprisingly, one can empirically show that for arbitrary $u$ and $v$, \ac{emd} is not \cnegdef{} on the circle despite its similarity to the line.

The reason that \ac{emd} is not \cnegdef{} on the circle is due to the use of the median in \refreq{transportCircle}.
If we approximate the median with the mean (guaranteed by Jensen's inequality to be within 1 standard deviation~\mycite{m-acooc}), then we obtain a \cnegdef{} approximation of \ac{emd}.
Note that substituting the mean in \refreq{transportCircle} yields
\begin{equation}
\footnotesize
\begin{split}
\emd(u,v) 
&\approx \int_0^{2\pi}\left|U(s)-\int_0^{2\pi}U(t)\dif t - \left(V(s)-\int_0^{2\pi}V(t)\dif t\right)\right|\dif s.
\end{split}
\end{equation}
which is a sum of \cnegdef{} kernels.
If the median can be expressed  by a function $h$ as $\alpha = h(u)-h(v)$ (perhaps only for certain families of distributions), then \ac{emd} is \cnegdef{}.

\section{Transportation on the $L_2$ Hypersphere} \label{sec:transportationL2Hypersphere}
Consider the class of ground distances of the form $\beta-K$, where $\beta$ is a positive constant and $K$ is \posdef{}.
This class of ground distances coincides with those implied by \cnegdef{} \EMDHat[\alpha] since we may note that Pele and Werman's \EMDHat[\alpha] is a special case formulation of \refreq{generalizedEMDHat} that uses $D(a,p)=\alpha\max\{D(a,b)\}$ for every $a,b \in X$.
If a point $p$ can be found or created such that $D(a,p)=\beta$ for each $a \in X \setminus \{p\}$ and $D$ is \cnegdef{}, then by Lemma~\ref{thm:defK} we can conclude that $D$ is of the form $\beta-K$ (in this case, $\beta=2\alpha\max D(a,b)-D(p,p)$).
A characterization of kernels of this form is given by Berg et al.\@~\mycitet{bcp-haos}.
If we add the condition that $D$ satisfies identity of indiscernibles, then a geometric interpretation of $D$ is readily forthcoming.
In particular, the image $\phi(X)$ from $K$'s feature mapping lies on the hypersphere of radius $\sqrt{\beta}$ in a Hilbert space centered on the point $\phi(p) = \vec{0}$.
This follows from the fact that $K(a,a) = \beta$ as a consequence of $D(a,a)=0$.
In other words, this subclass is comprised of normalized kernels and embeds into squared $L_2$ on the hypersphere.

\ifexpanded
Ground distances of this form have already appeared in the literature.
Rabin et al.\@~\mycitet{rdg-tdc} considered geodesic distances on the circle and used them for color image retrieval and color transfer between images.
The geodesic distance is equivalent to the angle between two vectors representing points on the circle, which can be computed from the arc-cosine of their dot product, which is of the form $\beta-K$ where $\beta=\pi/2$ and $K$ is the arc-sine of their dot product.
Zhang et al.\@~\mycitet{zmls-lfkctoccs} used a Euclidean ground distance in a high-dimensional space to compare SIFT descriptors for object and texture recognition in images.
However, they normalized the vectors comprising each set's support, effectively restricting their computations to distance between points on the hypersphere.
This study provided empirical evidence that \ac{emd} tends to be \cnegdef{} for this restricted case since no violations were found.

\else
Ground distances of this form have already appeared in the literature~\mycitet{rdg-tdc} and empirical evidence~\cite{zmls-lfkctoccs} suggests that \ac{emd} tends to be \cnegdef{} for this restricted case since no violations were found.
\fi
However, the result of Naor and Schechtman~\mycitet{ns-peinil1} states that \ac{emd} is indefinite on the $\{0,1\}^2 \subset \mathbb{R}^2$ grid with a Euclidean ground distance.
We can thus conclude that \ac{emd} is actually not \cnegdef{} for ground distances of the form $\beta-K$ in general since one can find a subspace of the hypersphere isometric to $\{0,1\}^2$.
Consequently, any ground distance must necessarily not include subspaces isometric to $\{0,1\}^2$ if there is any hope for \ac{emd} to be \cnegdef{}.
We do have one example, though, of a ground distance of this form---the discrete metric---where \ac{emd} is \cnegdef{}, and we hypothesize that ground distances close to discrete in form are also sufficient.
More formally, we hypothesize that there exists $\epsilon > 0$  such that if $K(x,x)=1$ for all $x \in X$ and $K(x,y) < \epsilon$ for all $x \ne y$, then \ac{emi} equipped with $K$ is \posdef{}.
We will now illustrate this notion with a method that transforms a ground distance into a nearly discrete form in order to yield \cnegdef{} \ac{emd}.

Under the following assumptions about the distribution of the sets under consideration for use with \ac{emd}, we may use Proposition~\ref{thm:inducedPDPT} to show that there exists a transformed ground distance of the form $\beta-K$ that yields \cnegdef{} \ac{emd}.
The assumptions that we make are that the sets are discrete, the collection of sets is finite, and that each pair of sets is disjoint.
Note that these assumptions form sufficient but not necessary conditions for the strategy that follows.
We also assume that $K$ strictly satisfies \refreq{posDefTriangle} for different $x,y$ but is not necessarily \posdef{}.
One may then infer that there exists a number $n_0$ for which $\ptransform{K}{n}$ is \posdef{} for $n \geq n_0$.

Let $X_1, X_2, \dots, X_n \in \mathcal{F}(X)$ be subsets of $X$ discretely supported with support cardinalities $s_i$, $i \in [1,n]$.
Let $K_i^j$ be the $s_{i} \times s_{j}$ kernel matrix computed between the elements of $X_i$ and $X_j$, and let 
\begin{equation}
F_i^j = \underset{f}{\arg\max}~\Vec{K_i^j}^\intercal\Vec{f}
\end{equation} 
be the $s_i \times s_j$ maximum-cost maximum-flow matrix computed between $X_i$ and $X_j$, where $\Vec{M}$ is the vectorization of the matrix $M$ made by concatenating columns.
Note that 
\begin{equation}
\emi(X_i,X_j) = \Vec{K_i^j}^\intercal\Vec{F_i^j}.
\end{equation}
Let $H_i^j$ be the Schur product of $F_i^j$ and $K_i^j$. 
Note that $H_i^i$ is diagonal for each $i$ as a consequence of \refreq{posDefTriangle}.
Additionally,
$\uflow_j^i = {\uflow_i^j}^\intercal$,
and
\begin{equation}
\emi(X_i,X_j) =\sum_{h=1}^{s_i}\sum_{k=1}^{s_j} {H_i^j}_{h,k}.
\end{equation}
By an application of the derived subsets kernel~\mycite{stc-kmpa}, we may deduce that \ac{emi} is \posdef{} if the kernel matrix $G_\uflow$, where the $(i,j)$-th block
$
G_\uflow(i,j) = \begin{bmatrix}\uflow_i^j
\end{bmatrix},
$
is \posdef{},
\ie if $\uflow: X^* \times X^* \to \mathbb{R}$ is a \posdef{} kernel, where $X^* = \bigcup_{i=1}^nX_i$.

There are several ways one may proceed to obtain \posdef{} \ac{emi}. 
One may transform $K$ and either keep or recompute the flow.
One may also transform $H$ or \ac{emi} itself.
Since the sets are disjoint and $K$ satisfies the conditions of Proposition~\ref{thm:inducedPDPT}, then $H$ and \ac{emi} satisfy the same conditions.
By Proposition~\ref{thm:inducedPDPT}, repeated transformation of $H$ or \ac{emi} will eventually become \posdef{}.
Transforming only $K$ is slightly more complicated to analyze, but one may note by similar reasoning used in the proof of Proposition~\ref{thm:inducedPDPT} that $G_H$ must eventually become \posdef{} since it converges to a diagonal matrix.
Note that we do not endorse this scheme for use with any ground distance, and we hypothesize that it is most appropriate for ground distances that are already normalized, \ie of the form $\beta-K$.

\ifIEEE
	\ifCLASSOPTIONcompsoc
	  \section*{Acknowledgments}
	\else
	  \section*{Acknowledgment}
	\fi
\else
	\acks
\fi

This work was funded in part by the Louisiana Space Consortium (LaSPACE) GSRA grant 89631 and NASA grant NNX10AI40H.

\ifIEEE
\ifCLASSOPTIONcaptionsoff
  \newpage
\fi
\fi



\ifIEEE
\bibliographystyle{myIEEEtran}
\else
\bibliographystyle{abbrvnat}
\fi
\bibliography{references}
%
%
%

\clearpage
\iffalse
We already know from Proposition~\ref{thm:emdIsSetDiff} that \ac{emd} with the discrete metric is equal to the symmetric difference.
In addition, we know that set intersection is \posdef{}~\mycite{btb-ghikir}.
By Lemma~\ref{thm:defK} with $x_0=\emptyset$ and $K$ as set intersection, we can conclude that the symmetric difference (as $D$) is \cnegdef{}.
As a result, we have shown that $\emd_{\text{0-1}}$ is \cnegdef{}.
More generally,
\begin{align}
\begin{split}
\emi_\alpha(A,B) =&\emdhat_{\alpha/2}(A,\emptyset)+\emdhat_{\alpha/2}(B,\emptyset)\\
&-\emdhat_{\alpha/2}(A,B)-\emdhat_{\alpha/2}(\emptyset,\emptyset),
\end{split}
\label{eq:defEMI}
\end{align}
and thus by Lemma~\ref{thm:defK} we can conclude that \EMDHat[\alpha/2] and SEMD$_{\alpha}$ are conditionally negative definite if EMI$_\alpha$ is \posdef{} (\ac{emd} can only be conditionally negative definite for sets of the same size).
Consequently, by Theorem~\ref{thm:expK} we have shown that \ac{emd} can in fact yield positive definite \ac{rbf} kernels.
\ifexpanded
Does our prior intuition regarding property inheritance match with these results?
The following proposition answers in the affirmative and is expanded in the proposition immediately after it.
\begin{proposition}
The discrete metric is conditionally negative definite.
\label{thm:discreteDefinite}
\end{proposition}
\begin{proof}
Let $c_1, \dots, c_n$ be a set of real numbers such that $\sum_{i=1}^nc_i=0$.
Let $x_1,\dots, x_n$ be members of an arbitrary space $X$ endowed with the discrete metric.
Then
\begin{equation}
\begin{split}
\sum_{i,j=1}^nc_ic_j\discrete(x_i,x_j)&=\sum_{i,j=1}^nc_ic_j-\sum_{i=1}^nc_i^2 \\
&=\left(\sum_{i=1}^nc_i\right)^2-\sum_{i=1}^nc_i^2\\
&=-\sum_{i=1}^nc_i^2\\
&\leq 0.
\end{split}
\end{equation}
Thus, the discrete metric is conditionally negative definite.
\end{proof}
\else
It is trivial to verify that the discrete metric is conditionally negative definite.
\fi

We have thus provided evidence that \ac{emd} can be conditionally negative definite, and the definiteness of the ground distance plays a significant role in determining this.
However, the discrete metric is only one data point and does not on its own suggest a trend nor imply that its case is not unique.
We will therefore present results for two classes of ground distances: conditionally negative definite kernels $D: X \times X \to \mathbb{R}$ of the form $g(x)+g(y)$ and the form $\beta-K$, where $g:X \to \mathbb{R}$ is an arbitrary function and $K:X \times X \to \mathbb{R}$ is a positive definite kernel.
Before we can propose relevant theorems, however, we must first introduce a new definition of \EMDHat[].
Whereas \EMDHat[\alpha] was introduced with the assumption of a thresholded ground distance, the definition we propose here has no such limitation.
Instead of stating that any unmapped mass computed from $\emd(A,B)$ must belong to $A \setminus B$ (assuming $A$ is larger) and choosing a threshold that determines the value of a single unmatched element, we choose to transport the excess mass to some sink $p \in X$.
To express this transportation of excess mass, we define
\begin{equation}
\begin{split}
\emdnot_{p}(A,B) =& \sum_{a \in A}\left(f_A(a)-\sum_{b \in B}f(a,b)\right)D(a,p)\\
&+ \sum_{b \in B}\left(f_B(b)-\sum_{a \in A}f(a,b)\right)D(b,p),
\end{split}
\end{equation}
where both series are shown to emphasize the symmetry of \EMDNot[p] even though at least one must be zero due to \ac{emd}'s constraints.
Strictly speaking $p$ does not have to be an element of $X$ (in which case the cost of transporting mass to $p$ could not be expressed by $D$).
However, it is convenient for the rest of this section to assume so unless otherwise noted.
We can thus define
\begin{align}
\begin{split}
\emdhat_{p}(A,B) &= \emd(A,B) + \emdnot_p(A,B).
\end{split}
\end{align}
Using \refreq{defEMI} as a guide, we can also define
\begin{align}
\begin{split}
\emi_p(A,B) = & \emdhat_p(A, \emptyset)+\emdhat_p(B,\emptyset)-\emdhat_p(A,B)\\
= &\sum_{a \in A}\sum_{b \in B}f(a,b)\left[D(a,p)+D(b,p)\right]\\
&-\emd(A,B).
\end{split}
\end{align}
We are thus prepared to present the following proposition, which proves that for all ground distances of the form $g(x)+g(y)$, \ac{emd} is conditionally negative definite (for sets of the same size).

\begin{proposition}
If there exists a function $g: X \to \mathbb{R}$ such that the ground distance $D(x,y) = g(x)+g(y)$, then \ac{emd} is conditionally negative definite on $\mathcal{P}(X)$ and EMI$_p$ is positive definite on $\mathcal{F}(X)$ and any choice of $p$ such that $g(p)\geq 0$.
\label{thm:constantEMD}
\end{proposition}
\begin{proof}
Let $f(a,b)$ be the minimum-cost maximum flow between $A$ and $B$.
By definition, 
\begin{equation}
\begin{split}
\emd(A,B) &= \sum_{a \in A}\sum_{b \in B}f(a,b)D(a,b)\\
&= \sum_{a \in A}\sum_{b \in B}f(a,b)\left[g(a)+g(b)\right]\\
&=\sum_{a \in A}\sum_{b \in B}f(a,b)g(a)+\sum_{a \in A}\sum_{b \in B}f(a,b)g(b).
\end{split}
\label{eq:constantFlow1}
\end{equation}
Applying constraints \refreq{outFlow}, \refreq{inFlow}, and \refreq{totalflow} with the fact that $\mu(A)=\mu(B)$ yields
\begin{equation}
\emd(A,B) = \sum_{a \in A}f_A(a)g(a)+\sum_{b \in B}f_B(b)g(b).
\label{eq:constantFlow2}
\end{equation}
In other words, 
\begin{equation}
\emd(A,B) = G(A)+G(B),
\label{eq:emdInherits}
\end{equation}
where 
\begin{equation}
G(A) = \sum_{a \in A}f_A(a)g(a).
\end{equation}
By Lemma~\ref{thm:constantK}, \ac{emd} is conditionally negative definite.

Now we shall show that EMI$_p$ is positive definite for $g(p) \geq 0$ by showing that it reduces to the min-kernel.
\begin{equation}
\begin{split}
\emi_p(A,B)=&\sum_{a \in A}\sum_{b \in B}f(a,b)\left[D(a,p)+D(b,p)\right]\\
&-\emd(A,B)\\
=&\sum_{a \in A}\sum_{b \in B}f(a,b)\left[g(a)+g(b)+2g(p)\right]\\
&- \sum_{a \in A}\sum_{b \in B}f(a,b)\left[g(a)+g(b)\right]\\
=&\sum_{a \in A}\sum_{b \in B}f(a,b)2g(p)
\end{split}
\end{equation}
Applying constraint \refreq{totalflow} yields
\begin{equation}
\emi_p(A,B)=2g(p)\min\left\{\mu(A),\mu(B)\right\}.
\end{equation}
Since the min-kernel is known to be positive definite and $g(p) \geq 0$, we can deduce that EMI$_p$ is positive definite on $\mathcal{F}(X)$.
\end{proof}
In \refreq{emdInherits}, \ac{emd} once again exhibits its tendency to inherit properties from the ground distance.
We may also note that EMI$_p$ reduces to just the min-kernel scaled by a function of $p$.
We may also use this as an example of a situation where $p$ does not have to be a member of $X$.
Suppose $f(x) \leq 0$ for each $x \in X$.
The proposition implies that we cannot find a positive definite EMI$_x$.
We may however arbitrarily define a hypothetical element $p$ and augmented function $h: (X \cup \left\{p\right\}) \to \mathbb{R}$ such that $h(p) > 0$ and $h(x)=g(x)$ for $x \ne p$.
The augmented function $h$ satisfies the conditions of the proposition, and EMI$_p$ will be positive definite.
Alternatively, of course, one could simply negate EMI$_x$.
Next we note a general result for \EMDHat[] before noting its implications for ground distances of the form $\beta-K$.

\begin{proposition}
If \ac{emd} is conditionally negative definite on $\mathcal{P}(X)$ and ground distance $D: X \times X \to \mathbb{R}$ and there exists $x_t \in X$ such that $D(x,x_t) \geq D(x,y)$ for all $x \in X \setminus \{x_t\}, y \in X$, then $\emdhat_{x_t}$ is conditionally negative definite on $\mathcal{F}(X)$.
\label{thm:extremePointEMD}
\end{proposition}
\begin{proof}
By definition, 
\begin{equation}
\emdhat_{x_t}(A,B) = \emd(A,B)+\emdnot_{x_t}(A,B).
\end{equation}
If $\mu(A)=\mu(B)$, then $\emdnot(A,B) = 0$.
Let $X_1,\dots, X_n$ be subsets of $X$, $S = \underset{i}{\max} \left\{\mu(X_i)\right\}$, and $\overline{X}_1, \dots, \overline{X}_n$ be singletons satisfying
\begin{align}
\supp\left(\overline{X}_i\right) &= \{x_t\}\\
f_{\overline{X}_i}(x_t) &= S-\mu\left(X_i\right).
\end{align}
We can create equal mass sets $X_i^\prime$ by adding $X_i$ and $\overline{X}_i$, \ie 
\begin{equation}
X^\prime_i = X_i+\overline{X}_i.
\end{equation}
Now consider $\emd(X_i^\prime, X_j^\prime)$ and its minimum-cost maximum flow.
Since $D(x,x_t) \geq D(x,y)$ for all $x \ne x_t$, the minimum-cost flow between $X_i \subseteq X^\prime_i$ and $X_j \subseteq X^\prime_j$ is the same as that computed for $\emd(X_i, X_j)$.
If the flows were not the same, we would arrive at a contradiction since this would imply $D(x,x_t) < D(x,y)$ for some $x, y$.
Any mass not mapped by the flow of $\emd(X_i, X_j)$ must consequently be transported to $x_t$ at cost
\begin{equation}
\begin{split}
C(X_i^\prime, X_j^\prime) =& \sum_{x \in X_i}\left(f_{X_i}(x)-\sum_{y \in X_j}f(x,y)\right)D(x,x_t)\\
&+\sum_{y \in X_j}\left(f_{X_j}(y)-\sum_{x \in X_i}f(x,y)\right)D(y,x_t)\\
&+\overline{C}(X_i^\prime, X_j^\prime),
\end{split}
\end{equation}
where
\begin{equation}
\overline{C}(X_i^\prime, X_j^\prime) = \min\left\{f_{\overline{X}_i}(x_t), f_{\overline{X}_j}(x_t)\right\}D(x_t,x_t).
\end{equation}
Thus, 
\begin{equation}
\emd(X_i^\prime, X_j^\prime) = \emd(X_i, X_j)+C(X_i^\prime, X_j^\prime).
\end{equation}
We may also note that
\begin{equation}
C(X_i^\prime, X_j^\prime) =\emdnot_{x_t}(X_i,X_j) + \overline{C}(X_i^\prime, X_j^\prime).
\end{equation}
In other words,
\begin{equation}
\emd(X_i^\prime, X_j^\prime) =\emdhat_{x_t}(X_i,X_j)+\overline{C}(X_i^\prime, X_j^\prime),
\end{equation}
 and consequently,
\begin{equation}
\emdhat_{x_t}(X_i,X_j) =\emd(X_i^\prime, X_j^\prime)-\overline{C}(X_i^\prime, X_j^\prime).
\end{equation}
Since $\mu(X_i^\prime)=S$ for all $i$, $\emd(X_i^\prime, X_j^\prime)$ is conditionally negative definite by hypothesis.
In addition, $\overline{C}$ is equivalent to the min-kernel and hence positive definite.
The opposite, $-\overline{C}$, is (conditionally) negative definite.
Therefore, as it is the sum of two conditionally negative definite kernels, $\emdhat_{x_t}$ is conditionally negative definite.
\end{proof}

The proposition relies upon the assumption that $x_t \in X$.
In fact, if one simply augmented every set with mass located at an arbitrary $x \in X$ until each set was the same size and computed \ac{emd} only between the augmented sets, then one would trivially have a conditionally negative definite version of \ac{emd} for sets of arbitrary mass.
However, this would not be an accurate reflection of \ac{emd} between the unnormalized sets unless the conditions of Proposition~\ref{thm:extremePointEMD} were met. 
Computing \EMDHat{} as proposed, wherein the excess mass is transported separately to an external sink, allows the original value of \ac{emd} to be used rather than compute a new flow that might distort the relationships between the sets and their respective elements.

If the conditions of Proposition~\ref{thm:extremePointEMD} are satisfied and
\begin{equation}
D(x,x_t)=D(y,x_t)=\alpha,
\label{eq:extremePointAlpha}
\end{equation}
then we may note that $\emdhat_{x_t}=\emdhat_\alpha$.
We may also deduce that the ground distance is of the form $\beta-K$, where $K$ is positive definite.
As stated, we assumed that $x_t \in X$.
However, considering a hypothetical $x_t \not\in X$ does not appear to be too great of a leap for this form of ground distance, as we can always create such an element if desired, as outlined in the following proposition.
\begin{proposition}
A conditionally negative definite kernel $D : X \times X \to \mathbb{R} $ is of the form $\beta- K$, where $K$ is a positive definite kernel, \iffy{} there exists an element $x_t$ and $\alpha \geq \beta/2$ such that $D^*:  (X \cup \{x_t\}) \times  (X \cup \{x_t\}) \to \mathbb{R}$ is conditionally negative definite, where
\begin{equation}
D^*(x,y) = \begin{cases}
D(x,y), & \text{if } x,y \in X,\\
0, & \text{if } x=y=x_t,\\
\alpha, & \text{otherwise}.
\end{cases}
\end{equation}
\label{thm:pointThreshold}
\end{proposition}
\begin{proof}
The ``if" part is easily shown by an application of Lemma~\ref{thm:defK} with $x_0=x_t$, which yields $K^*=2\alpha-D^*$.
Since $D$ operates on a subset of the domain of $D^*$, this implies that $K=2\alpha-D$ is also positive definite and that $D$ is conditionally negative definite and of the required form with $\beta=2\alpha$.

The ``only if'' part requires us to prove that $D^*$ inherits its definiteness from $D$.
Therefore, consider the kernel
\begin{equation}
\begin{split}
K^*(x,y) &= D^*(x,x_t)+D^*(y,x_t)-D(x,y)\\
&=2\alpha-D(x,y).
\end{split}
\end{equation}
If $\alpha \geq \beta/2$, then
\begin{equation}
K^*(x,y) = K(x,y)+\epsilon,
\end{equation}
where $\epsilon=2\alpha-\beta \geq 0$.
Considering the feature map $\phi: X \to H$ such that $H$ is a Hilbert space and $K(x,y) = \left<\phi(x),\phi(y)\right>$, we may note that $K^*$ is equivalent to adding a constant extra feature of magnitude $\sqrt{\epsilon}$ to each feature vector $\phi(x)$.
We can conclude then that $K^*$ is equivalent to an inner product (the dot product) in this augmented Hilbert space and is therefore positive definite. 
By Lemma~\ref{thm:defK}, $D^*$ is conditionally negative definite.

\end{proof}
The preceding thought process leads to the following conjecture, an easy proof of which is currently limited by the inability to assume \textit{a priori} that there exists $x_t \in X$ satisfying \refreq{extremePointAlpha}.
A characterization of kernels of the form $\beta-K$ is given by Berg et al.~\mycite{bcp-haos}.
\begin{conjecture}
If $D$ is a conditionally negative definite kernel of the form $\beta-K$, where $\beta$ is a constant and $K$ is a positive definite kernel, and \ac{emd} is conditionally negative definite with ground distance $D$ for sets of constant size, then $\emdhat_{\alpha}$ is conditionally negative definite with ground distance $D$ for $\alpha \geq \beta/2$ on $\mathcal{F}(X)$.
\label{thm:betaMKForm}
\end{conjecture}

The conditional negative definiteness of \ac{emd} is intimately linked to the positive definiteness of EMI.
In fact, if there exists a conditionally negative definite ground distance $D$ for which \ac{emd} is conditionally negative definite for sets of constant size, then there exists a positive definite ground distance $K$ such that \ac{emd} computed from a maximum-cost maximum flow with respect to $K$ is positive definite for sets of constant size.
Let us assume that $K(x,y) = D(x,p)+D(y,p)-D(x,y)$ as per Lemma~\ref{thm:defK}.
Then
\begin{equation}
\begin{split}
\mathit{EMI}_{p}&(A,B)\\
 = &\sum_{a \in A}\sum_{b \in B}f(a,b)\left[D(a,p)+D(b,p)\right]\\
&-\emd(A,B)\\
= &\sum_{a \in A}\sum_{b \in B}f(a,b)\left[D(a,p)+D(b,p)\right]\\
&-\sum_{a \in A}\sum_{b \in B}f(a,b)D(a,b)\\
= &\sum_{a \in A}\sum_{b \in B}f(a,b)\left[D(a,p)+D(b,p)\right]\\
&-\sum_{a \in A}\sum_{b \in B}f(a,b)\left[D(a,p)+D(b,p)-K(a,b)\right]\\
= &\sum_{a \in A}\sum_{b \in B}f(a,b)K(a,b).
\end{split}
\end{equation}
Thus we see that EMI is equivalent to \ac{emd} formulated as a maximization instead of a minimization (and specification of $p$ is not required).
In order to see that one may simply maximize with respect to $K$ and still yield the same flow as a minimization with respect to $D$, note from \refreq{constantFlow1} and \refreq{constantFlow2} with $g(x) = D(x,p)$ that 
\begin{equation}
\sum_{a \in A}\sum_{b \in B}f(a,b)\left[D(a,p)+D(b,p)\right]
\end{equation}
is constant for fixed $A$, $B$ and does not depend upon the flow.
Furthermore, if \EMDHat[p] is conditionally negative definite as we conjecture, then we can conclude that EMI is positive definite on $\mathcal{F}(X)$ for certain positive definite ground distances.
\textit{To summarize}, EMD formulated as a maximization tends to inherit positive definiteness on $\mathcal{F}(X)$, and EMD formulated as a minimization tends to inherit conditional negative definiteness on $\mathcal{P}(X)$. 
The restrictions on set sizes for conditionally negative definite ground distances has an intriguing parallel to the conditions imposed on $c_i$  in the definitions of conditionally positive and negative definite kernels.
These observations also shed some light on how the pyramid match kernel~\mycite{gd-pmkelsf}, composed of intersections, can provide a suitable positive definite alternative to \ac{emd}.
As one final example of \ac{emd} inheriting the form of its ground distance for sets of constant size $s$, we may note that \ac{emd} also inherits the form $\beta-K$,
\begin{equation}
\begin{split}
\emd(A,B) &= \sum_{a \in A}\sum_{b \in B}f(a,b)\left[\beta-K(a,b)\right]\\
&=\sum_{a \in A}\sum_{b \in B}f(a,b)\beta-\sum_{a \in A}\sum_{b \in B}f(a,b)K(a,b)\\
&=\beta s-\emi_\beta(A,B),
\end{split}
\end{equation}
although we cannot guarantee that $K=\emi_\beta$ is positive definite without a proof of Conjecture~\ref{thm:betaMKForm}.
It is interesting to note that the discrete metric is of the form $\beta-K$, which is easily shown by observing that $1-\discrete$ is positive definite.
In fact, the discrete kernel $1-\discrete$ is what $\ptransform{K}{k}$ converges to for any initial $K$ satisfying $2K(x,y)=K(x,x)+K(y,y)$ \iffy{} $x=y$.

We also see a sufficient condition for \ac{emd}'s conditional negative definiteness in that if the quantity $f(a,b)K(a,b)$ is positive definite (where care is taken to distinguish flows between different sets), then EMI is trivially positive definite by an application of the derived subsets kernel~\mycite{stc-kmpa}.
Likewise, it is sufficient if $f(a,b)$ is itself positive definite since Theorem~\ref{thm:productK} would imply that $f(a,b)K(a,b)$ is positive definite.
We may note that this condition is not necessary, however, since for the discrete ground distance any unit of mass not mapped by identity can be given an arbitrary destination as part of some minimum-cost flow, and we can therefore easily construct non-positive definite flows.

\section{A Set Theoretic Interpretation of \ac{emd}}
Suppose we are given a measure space $(X, \mu, D)$.
Furthermore, suppose we have two subsets $A$ and $B$ of this space, each with finite measure.
The general idea is to express basic set functions such as intersection and set difference in terms of \ac{emd}.
Typically, a binary relation is assumed in which two elements are either equal or not.
In other words, typically the discrete metric is used to determine whether elements in two sets are the same or not.
Sometimes, though, we can quantify their degree of equality more precisely than a simple ``equal" or ``not equal."
To highlight an issue with the binary approach, consider two sets of two-dimensional points where one is a slightly perturbed version of the other.
The points can individually be compared by the Euclidean distance or some other norm.
According to the strict definition of set intersection given by \refreq{intersection}, their intersection is empty despite the fact that they are clearly related by their elements.
\ac{emd} provides a natural method to compare the sets, although it is proportional to their difference rather than intersection.
We will show, however, that one can use \ac{emd} to get a measure of the point sets' intersection that incorporates the chosen norm for comparing elements.
We will see that \ac{emd} and subsequent related functions define soft generalizations or approximations of classic set operations.
First, however, let us observe how classic set operations (of two sets) can be formulated as special cases of \ac{emd}.
\begin{proposition}
Let $\emd_{\text{0-1}}$ be \ac{emd} equipped with the discrete metric as the ground distance. If $\mu(A) \leq \mu(B)$, then
\begin{equation}
\emd_{\text{0-1}}(A,B)=\emd_{\text{0-1}}(B,A)=\mu(A \setminus B).
\end{equation}
\label{thm:emdIsSetDiff}
\end{proposition}
\begin{proof}
Since the discrete metric satisfies identity of indiscernibles and triangle inequality, the minimum-cost maximum flow from an element $a \in A$ to itself $a \in B$ is 
\begin{equation}
f(a,a) = \min\left\{f_A(a),f_B(a)\right\}.
\label{eq:inclusionFlow}
\end{equation}
The equality cannot be greater than as this would violate either contraint~\refreq{outFlow} or~\refreq{inFlow}.
The equality also cannot be less than since we could find a cheaper flow by enforcing~\refreq{inclusionFlow}, thus contradicting the fact that $f$ is minimum-cost.
The remainder of the mass not mapped by inclusion can be transported arbitrarily and still be minimum-cost, and thus 
\begin{equation}
\emd_{\text{0-1}}(A,B) = \sum_{a\in A} \sum_{b \in B} f(a,b)\discrete(a,b),
\end{equation}
where $f$ is an arbitrary flow satisfying \refreq{outFlow}, \refreq{inFlow}, \refreq{totalflow}, and \refreq{inclusionFlow}.
The proposition directly follows from these constraints:
\begin{equation}\small
\begin{split}
\sum_{a\in A} \sum_{b \in B} f(a,b)\discrete(a,b) =&\sum_{a\in A} \sum_{b \in B} f(a,b)-\sum_{a \in A}f(a,a)\\
=&\mu(A)-\sum_{a \in A}\min\left\{f_A(a),f_B(a)\right\},
\end{split}
\end{equation}
combined with \refreq{mudef} and \refreq{intersection},
\begin{equation}
\begin{split}
=&\mu(A)-\sum_{a \in A}f_{A\cap B}(a)\\
=&\mu(A)-\mu(A \cap B)\\
=&\mu(A \setminus B).
\end{split}
\end{equation}
\end{proof}
From Proposition~\ref{thm:emdIsSetDiff} we can easily deduce that set intersection is simply the size of the smaller set minus \ac{emd} equipped with the discrete metric, \ie
\begin{equation}
\mu(A \cap B) = \min\{\mu(A), \mu(B)\}-\emd_{\text{0-1}}(A,B).
\label{eq:muIntersection}
\end{equation}
We can also easily deduce the union using the well-known relation between intersection and union:
\begin{equation}
\begin{split}
\mu(A \cup B) =& \mu(A)+\mu(B)-\mu(A \cap B)\\
=&\mu(A)+\mu(B)-\min\{\mu(A), \mu(B)\}\\&+\emd_{\text{0-1}}(A,B)\\
=&\max\{\mu(A), \mu(B)\}+\emd_{\text{0-1}}(A,B).
\end{split}
\end{equation}
Finally, we can determine the larger of the set differences and the symmetric difference:
\begin{align}
\begin{split}
\mu(B \setminus A) &= \mu(B)- \mu(A \cap B)\\
&=\emd_{\text{0-1}}(A,B)+|\mu(A)- \mu(B)|,
\label{eq:muLargerDif}
\end{split}\\
\begin{split}
\mu(A \triangle B) &= \mu(A \cup B) - \mu(A \cap B)\\
&=2\emd_{\text{0-1}}(A,B)+|\mu(A)- \mu(B)|.
\label{eq:muSymmetricDif}
\end{split}
\end{align}
Observing these equations, it is relatively easy to extrapolate generalizations of these set operations by changing the ground distance.
We assume, however, that the ground distance is bounded by some threshold $t$, which indicates the point at which two elements are considered ``not equal."
For practical purposes, if the measure space $(X, \mu, D)$ is unbounded, then it suffices to choose $t$ greater than or equal to the diameter of the dataset.
Alternatively, one may apply a threshold to the ground distance to produce $D_t(x,y) = \min\left\{t, D(x,y)\right\}$.
Note, however, that applying a threshold will change the value of \ac{emd} and possibly affect it and the ground distance's properties.

Observing \refreq{muIntersection}, we may note the first result of this interpretation of \ac{emd}: a positive definite kernel~\mycite{btb-ghikir}~\mycite{obv-bkbsim} derived directly from \ac{emd}!
Shortly we will show that this implies that $\emd_{\text{0-1}}$ is conditionally negative definite and how this result can be generalized to other ground distances.
First, however, we introduce the generalized set operations \EMDHat{}, SEMD, and EMI and show how they fit into existing research on metrics and kernels.
Table~\ref{tab:setOps} summarizes the basic set operation to which each generalization corresponds.

\subsection{Metrics}
This subsection discusses \ac{emd} based functions that measure the differences between sets. 
Some of these functions are already known but have never been incorporated into a set-based interpretation.
We dub them metrics because it has been proven through various sources that they are metrics for metric ground distances.

\subsubsection{Symmetric EMD}
Pele and Werman introduced \EMDHat{} as a means to calculate \ac{emd} between unnormalized histograms~\mycite{pw-lthmism} for use in nearest neighbor calculations and image retrieval.
Their metric, however, readily fits into our set-based framework as a measure of the set difference of the smaller set from the larger.
After all, its form is almost identical to that of \refreq{muLargerDif}:
\begin{equation}
\emdhat_{\alpha}(A,B) = \emd(A,B)+\alpha |\mu(A)-\mu(B)|.
\end{equation}
\EMDHat{} also readily yields a generalization of the symmetric difference \refreq{muSymmetricDif}, which we call Symmetric Earth Mover's Distance (SEMD):
\begin{align}
\semd_{\alpha}(A,B) &= \emdhat_{\alpha}(A,B)+\emd(A,B)\\
&= 2\emdhat_{\alpha/2}(A,B).\label{eq:doubleEMDHat}
\end{align}
In both cases, $\alpha$ is a scaling factor that alters the threshold used for element comparison and the cost of an unmatched element.
Effectively, $\alpha$ denotes the value of one unit of measure.
Pele and Werman~\mycite{pw-lthmism} proved that \EMDHat{} is a metric for $\alpha \geq t/2$, assuming the ground distance is also a metric.
Note that our representation of \EMDHat{} differs slightly from the original in that we have merged $t$ into the value of $\alpha$.
With a discrete ground distance and $\alpha=1$, both \EMDHat{} and SEMD reduce to their respective classical set operations.

\begin{table}
\centering
\caption{A list of earth mover functions and their corresponding mimicked set operations between two sets $A$ and $B$. Here we assume $\mu(A)\leq\mu(B)$.}
\begin{tabular}{| c | c|}
\hlx{hv}
\bf Function & \bf Set Operation\\
\hlx{vhv}
$\emd$ & $A \setminus B$\\
$\emdhat$ & $B \setminus A$\\
$\semd$ & $A \triangle B$\\
$\emi$ & $A \cap B$\\
\hlx{h}
\end{tabular}
\label{tab:setOps}
\end{table}

\subsubsection{Jaccard EMD}
We can also generalize any function that can be expressed in terms of basic set operations. 
One such function is the Jaccard distance.
The Jaccard distance is the complement of the Jaccard index, both of which have been known under a variety of names including Steinhaus (for arbitrary measure), biotope, Tanimoto, and Marczewski-Steinhaus~\mycite{dd-eod}.
The Jaccard index $J$ and distance $J_D$ between two sets $A$ and $B$ are given by 
\begin{align}
\mathit{J}(A,B) &= \frac{\mu(A \cap B)}{\mu(A \cup B)},\\
\mathit{J_D}(A,B) &= 1-J(A,B)=\frac{\mu(A \triangle B)}{\mu(A \cup B)}. \label{eq:jaccardDist}
\end{align}
If $A=B=\emptyset$, then $J(A,B)$ is defined to be 1.
We can thus define the Jaccard Earth Mover's Distance (JEMD) as a generalization of \refreq{jaccardDist} by noting that $\mu(A \cup B) =\frac{1}{2}\left[\mu(A)+\mu(B)+\mu(A \triangle B)\right]$ and replacing the symmetric difference with SEMD:
\begin{align}
\jemd_\alpha(A,B) 
&= \frac{2\semd_\alpha(A,B)}{\alpha \left[\mu(A)+\mu(B)\right]+\semd_\alpha(A,B)}.
\end{align}
The metric properties of JEMD are easily derived by noting that it is the biotope transform (or $p$-smoothing transform with $p=\emptyset$)~\mycite{dd-eod} of SEMD and is therefore a metric whenever SEMD is a metric.
Note that through \refreq{doubleEMDHat} it is also equivalent to the biotope transform of $\widehat{\text{EMD}}_{\alpha/2}$.
With a discrete ground distance and $\alpha=1$, JEMD is equivalent to the Jaccard distance.
JEMD and even \EMDHat{} were originally proposed under different names for point sets by Ramon and Bruynooghe~\mycite{rb-ptcmbps}, although neither a connection to \ac{emd} nor the Jaccard index was acknowledged.
JEMD was independently discovered and generalized to multisets of arbitrary measure by Gardner et al.~\mycite{CVPR2014}, where it outperformed \ac{emd} in a nearest neighbor setting.
Lower and upper bounds for JEMD can be easily computed~\mycite{CVPR2014}:
\begin{equation}
\frac{|\mu(A)-\mu(B)|}{\max\left\{\mu(A),\mu(B)\right\}} \leq \jemd(A,B) \leq \mathit{J_D}(A,B).
\label{eq:jemdBounds}
\end{equation}

\subsection{Kernels}
Now we propose kernel forms derived from \ac{emd}'s connection with set intersection.
In particular, we propose Earth Mover's Intersection (EMI) and Earth Mover's Jaccard Index (EMJI).
Unlike the preceding metrics (which are proved to satisfy metric properties), the proposed kernels are not proved to be positive definite in the general case.

\subsubsection{Earth Mover's Intersection}
Similar to the derivations of SEMD and JEMD, the generalized set intersection EMI is relatively easy to determine:
\begin{equation}
\emi_\alpha(A,B) = \alpha \min\left\{\mu(A),\mu(B)\right\} - \emd(A,B).
\end{equation}
This can also be used to define a generalized pyramid match kernel~\mycite{gd-pmkelsf}, which uses histogram intersections at various resolutions to approximate a partial matching between sets of image features.
An exploration of this idea is beyond the scope of this paper.

\subsubsection{Earth Mover's Jaccard Index}
Given the definition of the Jaccard distance in terms of \ac{emd}, it should come as little surprise that we can define a generalization of the Jaccard index as well:
\begin{align}
\emji_\alpha(A,B) &= \frac{\emi_\alpha(A,B)}{\alpha \max\left\{\mu(A),\mu(B)\right\}+\emd(A,B)}\\
&=\frac{\alpha \left[\mu(A)+\mu(B)\right]-\semd_\alpha(A,B)}{\alpha\left[\mu(A)+\mu(B)\right]+\semd_\alpha(A,B)}\\
&=1-\jemd_\alpha(A,B)
\end{align}
One may note that it maintains the relationship implied by \refreq{jaccardDist} regardless of the ground distance.

\section{Experimental Evaluation}
Experiments were performed on three datasets to assess the performance of the proposed kernel forms for \ac{emd} in \ac{svm}-based classification.
Note that the purpose of the experiments is to examine the kernels' quality for classification, not design nor compete with state-of-the-art recognition systems.
In particular, we explored texture recognition, object category classification, and hand posture recognition.
In each case, we evaluated classification performance for both normalized and unnormalized sets, where a set is normalized by scaling its total mass to be 1.
Each of the experiments continues to rely on the assumption that \ac{emd} is conditionally negative definite for a Euclidean ground distance. 
For object category classification, we assume the same for a squared Euclidean ground distance. 
In addition, parts of the experiments rely on the unproven assumption that applying a threshold is a conditionally negative-definite-preserving transformation.
Our research did not discover any existing theorems regarding the following conjecture, which states the assumption more formally.
Pele and Werman show that a similar theorem for metrics is true~\mycite{pw-fremd}.
\begin{conjecture}
If $D:X \times X \to \mathbb{R}$ is a conditionally negative definite kernel and $t \in \mathbb{R}$, then the thresholded kernel
\begin{equation}
D_t(x,y) = \min\left\{ D(x,y), t\right\}
\end{equation}
is also conditionally negative definite.
\end{conjecture}
Multiple thresholds for each chosen ground distance were tested.
For all experiments, $\alpha=1$ may be presumed. 
A chosen threshold of $t$ denotes that only features within a $t$-radius hypersphere are considered comparable.
Pele and Werman also showed that thresholds enable faster computation of \ac{emd}~\mycite{pw-fremd}, and thus their algorithm was used to calculate \ac{emd} and its derivative kernels in each experiment.

Kernel-based \acp{svm} serve as the primary means of classification.
A positive or negative classification of a query $x$ is given by the sign of
\begin{equation}
g(x)=\sum_i \gamma_iK(s_i,x)+b,
\end{equation}
where $\gamma_i$ is the learned weight of the support vector $s_i$ and $b$ is a learned threshold.
One-versus-all \acp{svm} provided by MATLAB's Statistics Toolbox Release 2014a were used in each experiment.
In addition to EMI and EMJI, we considered generalized \ac{rbf} kernels for each of \ac{emd}, \EMDHat{}, and JEMD. 
The scaling factor $u$ of each \ac{rbf} kernel was calculated as the inverse of the average value of $D$ between training samples, which was reported by Zhang et al.~\mycite{zmls-lfkctoccs} as a serviceable substitute for a more rigorously obtained value via some other, more time-consuming method such as cross validation.

\begin{table*}
\centering
\caption{Accuracies and standard deviations for texture recognition.}
\begin{tabular}{| c ||c c c c c | c c c c c |}
\hlx{hv[2-5,7-10]} 
& \multicolumn{5}{|c|}{\bf Unnormalized}&  \multicolumn{5}{c|}{\bf Normalized} \\
\hlx{hv}
Threshold & $\emd$ & $\emdhat$ & $\jemd$ & $\emi$ & $\emji$ & $\emd$ & $\emdhat$ & $\jemd$ & $\emi$ & $\emji$\\
\hlx{vhv}
0.5 & 57.45 & 51.95 & 62.55 & 61.95 & \textbf{64.00} & 70.05 & 70.05 & 69.75 & 68.40 & 71.10\\
  & $\pm 1.28 $ & $\pm 0.65 $ & $\pm 0.51 $ & $\pm 0.86 $ & $\mathbf{\pm 1.08}$ & $\pm 6.34 $ & $\pm 6.34 $ & $\pm 5.91 $ & $\pm 5.56 $ & $\pm 5.88$ \\
1 & 34.90 & 50.30 & 55.70 & 61.65 & 58.30 & 72.40 & 72.40 & 72.70 & 72.10 & \textbf{73.05}\\
  & $\pm 1.43 $ & $\pm 0.54 $ & $\pm 0.60 $ & $\pm 0.82 $ & $\pm 0.62 $ & $\pm 0.89 $ & $\pm 0.89 $ & $\pm 0.72 $ & $\pm 1.43 $ & $\mathbf{\pm 0.97}$\\
1.4142 & 34.40 & 45.10 & 51.45 & 60.85 & 55.55 & 68.35 & 68.35 & 67.30 & 68.15 & 67.70\\
  & $\pm 0.52 $ & $\pm 0.82 $ & $\pm 0.67 $ & $\pm 0.63 $ & $\pm 0.62$ & $\pm 7.35 $ & $\pm 7.35 $ & $\pm 6.62 $ & $\pm 7.39 $ & $\pm 7.11$ \\
\hlx{h}
\end{tabular}
\label{tab:texture}
\vskip 5 pt
\caption{Accuracies and standard deviations for object category classification.}
\begin{tabular}{| c ||c c c c c | c c c c c |}
\hlx{hv[2-5,7-10]} 
& \multicolumn{5}{|c|}{\bf Unnormalized}&  \multicolumn{5}{c|}{\bf Normalized} \\
\hlx{hv}
Threshold & $\emd$ & $\emdhat$ & $\jemd$ & $\emi$ & $\emji$ & $\emd$ & $\emdhat$ & $\jemd$ & $\emi$ & $\emji$\\
\hlx{vhv}
0.5 & 13.77 & 45.81 & 44.69 & \textbf{49.14} & 47.37 & 46.69 & 46.69 & 44.53 & 46.64 & \textbf{46.94}\\
 & $\pm 0.86 $ & $\pm 1.54 $ & $\pm 1.12 $ & $\mathbf{\pm 1.30} $ & $\pm 1.43$ & $\pm 1.26 $ & $\pm 1.26 $ & $\pm 1.13 $ & $\pm 0.45 $ & $\mathbf{\pm 1.21}$ \\
1 & 9.40 & 42.92 & 44.33 & 44.61 & 45.29 & 46.10 & 46.10 & 45.57 & 45.17 & 46.27\\
 & $\pm 0.93 $ & $\pm 0.85 $ & $\pm 0.78 $ & $\pm 0.98 $ & $\pm 1.30$ & $\pm 1.28 $ & $\pm 1.28 $ & $\pm 1.05 $ & $\pm 1.22 $ & $\pm 1.10$ \\
2 & 8.49 & 39.64 & 42.86 & 42.06 & 43.13 & 45.94 & 45.94 & 45.87 & 44.99 & 45.64\\
 & $\pm 0.76 $ & $\pm 1.37 $ & $\pm 1.22 $ & $\pm 1.07 $ & $\pm 1.40$ & $\pm 1.34 $ & $\pm 1.34 $ & $\pm 1.31 $ & $\pm 1.38 $ & $\pm 1.33$ \\
\hlx{h}
\end{tabular}
\label{tab:object}
\vskip 5 pt
\caption{Accuracies and standard deviations for posture recognition.}
\begin{tabular}{| c ||c c c c c | c c c c c |}
\hlx{hv[2-5,7-10]} 
& \multicolumn{5}{|c|}{\bf Unnormalized}&  \multicolumn{5}{c|}{\bf Normalized} \\
\hlx{hv}
Threshold & $\emd$ & $\emdhat$ & $\jemd$ & $\emi$ & $\emji$ & $\emd$ & $\emdhat$ & $\jemd$ & $\emi$ & $\emji$\\
\hlx{vhv}
25 & 37.20 & 80.87 & 79.33 & 76.33 & 80.44 & 80.31 & 80.31 & 79.69 & 78.60 & 79.96\\
 & $\pm 16.56 $ & $\pm 11.11 $ & $\pm 12.37 $ & $\pm 14.59 $ & $\pm 12.03 $ & $\pm 10.45 $ & $\pm 10.45 $ & $\pm 10.51 $ & $\pm 11.14 $ & $\pm 10.54$\\
50 & 38.96 & 90.91 & 89.42 & 85.56 & 89.13 & 89.64 & 89.64 & 88.42 & 82.96 & 89.49\\
  & $\pm 18.65 $ & $\pm 12.03 $ & $\pm 12.49 $ & $\pm 13.80 $ & $\pm 12.89 $ & $\pm 11.44 $ & $\pm 11.44 $ & $\pm 11.55 $ & $\pm 17.45 $ & $\pm 11.46$\\
100 & 32.80 & 95.02 & 93.53 & 87.42 & 94.58 & 93.53 & 93.53 & 92.56 & 93.82 & 93.42\\
  & $\pm 20.22 $ & $\pm 6.37 $ & $\pm 7.80 $ & $\pm 14.26 $ & $\pm 6.65 $ & $\pm 6.54 $ & $\pm 6.54 $ & $\pm 7.38 $ & $\pm 5.03 $ & $\pm 6.18$\\
150 & 28.96 & 95.47 & 95.13 & 86.78 & \textbf{95.82} & 95.13 & 95.13 & 94.78 & \textbf{96.64} & 95.36\\
  & $\pm 22.31 $ & $\pm 6.40 $ & $\pm 6.73 $ & $\pm 15.21 $ & $\mathbf{\pm 6.00} $ & $\pm 5.56 $ & $\pm 5.56 $ & $\pm 5.88 $ & $\mathbf{\pm 2.27 }$ & $\pm 5.30$\\
200 & 29.73 & 95.09 & 94.07 & 86.49 & 95.44 & 93.82 & 93.82 & 93.40 & 94.33 & 94.36\\
  & $\pm 18.65 $ & $\pm 6.73 $ & $\pm 7.47 $ & $\pm 15.20 $ & $\pm 6.20 $ & $\pm 6.75 $ & $\pm 6.75 $ & $\pm 7.03 $ & $\pm 6.56 $ & $\pm 6.66$\\
\hlx{h}
\end{tabular}
\label{tab:posture}
\end{table*}
\subsection{Texture Recognition}
The KTH-TIPS database~\mycite{hcfe-osrwcmc} is comprised of 10 texture classes under varying scale, pose, and illumination with 81 instances per class.
We adopted much of the experimental design of Zhang et al.~\mycite{zmls-lfkctoccs}, extracting features using Vedaldi and Fulkerson's implementation of the SIFT descriptor~\mycite{vf-vlfeat} with the Difference of Gaussians region detector.
The SIFT descriptor~\mycite{l-difsik} computes an $N$-bin histogram of image gradient orientations for an $M \times M$ grid of samples in the region of interest, resulting in an $M\times M\times N$ dimensional vector.
The resulting vectors are scaled to have a Euclidean norm of 1 to reduce the influence of illumination changes.
The descriptors are then clustered using a $k$-means algorithm (with $k=40$) to produce so-called image \textit{signatures}.
Note that the chief difference between our SIFT descriptors and those used by Zhang et al.~\mycite{zmls-lfkctoccs} is the estimation of the dominant gradient orientation.
The Euclidean distance is used as the ground distance between SIFT descriptors both for clustering and classification (\ie \ac{emd} calculation).
Since all of the SIFT descriptors are normalized with non-negative components,  we can consider the Euclidean ground distance to operate on the set $X = (Y \times \left\{0\right\}) \cup Z$ where 
\begin{align}
Y &= \left\{ \vec{x} \suchthat \vec{x}^\intercal\vec{x} = 1, \vec{x} \in \mathbb{R}^{M\times M \times N}\right\},\\
Z &=\left\{0\right\}^{M\times M \times N} \times \left\{1\right\}.
\end{align}
It is relatively easy to show that the diameter of $Y$ is $\sqrt{2}$.
In addition, $x_t \in Z$ is exactly a distance of $\sqrt{2}$ from every element of $Y$.
Assuming \ac{emd} is conditionally negative definite for this ground distance and thresholded versions of it, we may thus ensure the positive definiteness of EMI$_{x_t}$, EMJI$_{x_t}$, and the generalized \acp{rbf} by applying Proposition~\ref{thm:extremePointEMD} and following the proven implications outlined in Figure~\ref{fig:implications}.
Thresholds of 0.5, 1, and $\sqrt{2}$ (unthresholded) are used to examine the effect that the threshold has on performance.

For each class and threshold, 5 pairs of disjoint sets of 40 training samples and 40 testing samples were randomly selected, yielding 5 training sets and 5 corresponding test sets. 
Given the balanced representation among classes, accuracy was deemed an appropriate tool to measure the classification performance.
The results for each training and test pair were then averaged and are reported in Table~\ref{tab:texture}

\subsection{Object Category Classification}
The methodology used for object category classification is similar to that of texture recognition, primarily differing in the dataset and choice of feature extraction.
The Caltech-101 dataset~\mycite{ffp-osloc} is comprised of 101 categories (\eg face, car, etc.) with varied presentation as well as an extra category for unknown (BACKGROUND\_Google).
Instead of SIFT descriptors, the PHOW descriptor implemented by Vedaldi and Fulkerson~\mycite{vf-vlfeat} is used to represent images.
At a high level, the PHOW descriptor is a dense SIFT extractor (the regions of interest are densely sampled in a grid) that operates on multiple color channels instead of just grayscale. 
The squared Euclidean distance serves as the ground distance between the descriptors, and thresholds of 0.5, 1, and 2 (unthresholded) are applied to it.
The similarity of the PHOW and SIFT descriptors allows the same $X$, $Y$, and $Z$ from texture recognition to be used to derive positive definite kernels, provided that \ac{emd} is conditionally negative definite for a squared Euclidean ground distance.
For each threshold and each class with the exception of BACKGROUND\_Google, 5 pairs of disjoint sets of 15 training samples and 15 test samples were randomly selected (as opposed to 40 for texture recognition) for classification, the results of which are reported in Table~\ref{tab:object}.

\subsection{Posture Recognition}
The dataset used for posture recognition is comprised of 5 hand postures captured for 12 users using a Vicon motion capture camera system and a glove with attached infrared markers on certain joints.
A rigid pattern on the back of the glove is used to establish a translation and rotation invariant local coordinate system for the hand's markers.
The five postures in the dataset are fist, pointing with one finger, pointing with two fingers, stop (hand flat), and grab (fingers curled), each represented in the dataset by
several hundred instances per user captured as parts of intermittent streams where the user held the posture for a short time.
Instances are variable-size (due to occlusion) unordered sets of 3D points and are preprocessed by transforming to local coordinates and removing all markers more than 200 mm from the origin.
Due to the streaming nature of the data capture, it is likely that for an instance of a given user there will be a duplicate or near duplicate within the user's dataset, \ie the postures for a given user are highly correlated.
Therefore, we adopted a leave-one-user-out evaluation strategy.
In addition, this strategy allows us to measure the ability of the classifier to generalize to users it has not seen before, just as it would need to do in practice.
The Euclidean distance (on $\mathbb{R}^3$) is again chosen as the ground distance, and $x_t$ is chosen to be a sufficiently distant point so as to enable application of Proposition~\ref{thm:extremePointEMD}.
Thresholds of 25, 50, 100, 150, and 200 are applied to the ground distance, and we may note that in this scenario they are easier to interpret as they correspond to straight-line distances measured in millimeters. 
Effectively, if threshold $t$ is used, we are stating that only points within $t$ mm of each other can possibly represent the same physical marker.

For each threshold, class, and user, 75 instances were randomly chosen without replacement.
Since we performed a leave-one-user-out evaluation on 12 users, this yielded 12 pairs of disjoint training/test sets. 
The results of the classification are contained in Table~\ref{tab:posture}.
This dataset served as the original motivation for this work, and thus we place greater emphasis on the actual accuracy obtained, noting that the results outperform the best method presented in~\mycite{SMC2014} by nearly 10\% and with lower deviation.


\subsection{Discussion}
The first thing we may note is that all kernels, regardless of whether the sets were normalized, perform significantly better than the default \ac{emd} \ac{rbf} kernel on unnormalized sets, which appears to completely fall apart.
This observation reinforces the intuition that \ac{emd} can only be conditionally negative definite for sets of equivalent mass.
On the other hand, although the number of samples is small, there appears to be a general trend wherein \ac{emd} improves with smaller thresholds.
Presumably, this trend is due to \ac{emd} approaching conditional negative definiteness through increasingly binary element comparisons.
In other words, the error in \ac{emd} as an approximation of set difference is becoming smaller and smaller.
At some point, though, one can expect the improvement to drop off as the measure for equality becomes too sharp to quantify any similarity beyond identity.
This behavior may explain the slight decrease in accuracy observed when moving from a threshold of 50 to 25 mm in Table~\ref{tab:posture}.
One may also note a sharp drop in accuracy for the rest of the kernels at this same change in threshold.

The choice of threshold for the other kernels also clearly plays a role, although a trend is not nearly as uniform across experiments as it is for \ac{emd}. 
For posture recognition, the threshold yielding peak performance is relatively high at 150, whereas for the others it tends to be low.
The terms high and low in this sense are misleading, however, without more information about the distribution of distances. 
Choosing the optimal threshold, however, is clearly an important task for maximizing classification accuracy.
An open question is whether there is one optimal threshold or if there exist local extrema.
A more detailed analysis of the threshold's role is beyond the scope of this paper, but may be worth exploring.

If we consider each choice of normalization a separate experiment, then EMJI performed best in 4 out of 6 experiments (the other two of which belong to EMI).
If we consider each threshold and choice of normalization a separate experiment, then there are a total of 22 experiments.
EMJI performed best in 10 of these, \EMDHat[] in 7, and EMI in 5.
The standard \ac{emd} kernel performed best in 4 experiments that form a subset of \EMDHat[]'s 7 best.
Beyond the obvious disadvantage of \ac{emd} for unnormalized sets, one cannot say that any kernel is necessarily better than another from the data.
Note, however, that whereas \ac{emd}, \EMDHat[], and JEMD have the \ac{rbf} parameter $u$ chosen to ideally boost their accuracy, EMI and EMJI have no such parameter other than $\alpha$, which was chosen neutrally to be 1 for all relevant kernels.
Regardless of this fact, both EMJI and EMI still outperform the \ac{rbf} kernels in the majority of experiments. 
The apparent slight advantage of EMJI in particular is intriguing, although a full exploration of EMJI and the more general transform $\ptransform{K}{}$ defined by \refreq{pdpT} is beyond the scope of the paper.
An open question is whether nesting transformations can yield any benefit.
\fi

\iftrue
\else
\begin{IEEEbiographynophoto}{Andrew Gardner}
received the B.S. degree in computer science from Louisiana Tech University, Ruston, LA, USA, in 2012 and the M.S. degree in computer science from Louisiana Tech University in 2015.
In the summer of 2015, he worked at Lawrence Livermore National Laboratory as a high energy density physics summer student intern.
He is currently pursuing a Ph.D. degree in computational analysis and modeling at Louisiana Tech University.
His research interests include machine learning, graph theory, and transportation theory.
\end{IEEEbiographynophoto}
\begin{IEEEbiographynophoto}{Christian A. Duncan} earned the B.S. degree in computer science from
Johns Hopkins University, Baltimore, MD, USA, in 1994, the
M.S.E. degree in computer science from Johns Hopkins University in
1994, and the Ph.D. degree in computer science from Johns Hopkins
University in 1999.

From 1999 to 2000, he was a Postdoctoral Fellow at the
Max-Planck-Institut f\"ur Informatik, Saarbr\"ucken, Germany. From
2000 to 2006, he was an Assistant Professor of Computer Science at the
University of Miami, Miami, FL, USA.  From 2006 to 2012, he was an
Assistant Professor of Computer Science at Louisiana Tech University,
Ruston, LA, USA.  He is currently an Associate Professor of Computer
Science at Quinnipiac University, Hamden, CT, USA.  His research
interests and expertise lie in the area of developing and analyzing
algorithms for geometric and graph visualization problems.
\end{IEEEbiographynophoto}
\begin{IEEEbiographynophoto}{Jinko Kanno}
 earned her Ph.D in mathematics, graph theory, at Louisiana State University in 2003; before coming to the United States, she also earned her Masters in mathematics, topology in low dimensional manifolds and knot theory. Her research interests are in structural graph theory such as splitter theorems, in topological graph theory, and in applying mathematics towards problems in computer science and engineering. She is currently an associate professor of mathematics and statistics at Louisiana Tech University.
\end{IEEEbiographynophoto}
\newpage
\begin{IEEEbiographynophoto}{Rastko R. Selmic}
is an AT\&T Professor of Electrical Engineering at Louisiana Tech University. He was a Research Fellow at the Air Force Research Laboratory (AFRL) during summer 2015, 2008, 2007. He obtained the Bachelor’s degree in Electrical Engineering at the University of Belgrade in 1994, and then received his Master’s and Ph.D. degrees in Electrical Engineering at the University of Texas at Arlington in 1997 and 2000, respectively. From 1997 to 2002 he was a Lead DSP Systems Engineer at Signalogic, Inc. in Dallas, Texas. From 2002 to 2008 he was an Assistant Professor and from 2008 to 2014 an Associate Professor of Electrical Engineering at Louisiana Tech University.

Dr. Selmic's current research interests include smart sensors and actuators, cooperative sensing and control, gesture-based computing and control, and Micro-Aerial Vehicles (MAV) control. He is the author/co-author of one U.S. patent with four additional reports of invention, 65 journal and conference papers, 3 book chapters, and the book Neuro-Fuzzy Control of Industrial Systems with Actuator Nonlinearities, SIAM Press, Philadelphia, PA, 2002. Dr. Selmic received 2009 IFM Award for Outstanding Publication, ARRI Invention Award in 2000, the first prize at the IEEE Fort Worth Section Graduate Paper Contest in 1999, was a finalist for the Best Paper Award at IEEE International Conference on Control Applications in 1998, ARRI Best Paper Award in 1997 and Scholarship from Signalogic Inc. in 1996. He served as a proposal reviewer for National Science Foundation and Atlantic Innovation Fund, Canada. Dr. Selmic served as an Associate Editor for IEEE Transactions on Neural Networks and currently serves as an Associate Editor for IEEE Transactions on Cybernetics. He is a Senior Member of IEEE and Sigma Xi. He consulted for numerous companies including Andrew Corporation (Richardson, TX), Intelligent Automation Inc. (Rockville, MD), Davis Technologies International, Inc. (Dallas, TX), and American GNC Corp. (Simi Valley, CA).
\end{IEEEbiographynophoto}



\fi

\end{document}